\documentclass[11pt]{article}

\usepackage[margin = 1in]{geometry}

\usepackage{authblk}

\usepackage{microtype}
\usepackage{graphicx}
\usepackage{subfigure}
\usepackage{booktabs} 
\usepackage{natbib}
\usepackage{hyperref}
\usepackage{multicol}
\usepackage{multirow}  
\usepackage{colortbl}  
\usepackage{graphicx}
\usepackage{booktabs} 
\usepackage{pifont}   
\usepackage{xcolor} 
\newcommand{\cmark}{\textcolor{green}{\ding{51}}}  
\newcommand{\xmark}{\textcolor{red}{\ding{55}}} 
\newcommand{\E}{\mathbb{E}}

\usepackage{algorithm,algorithmic}

\usepackage{amsmath}
\usepackage{amssymb}
\usepackage{mathtools}
\usepackage{amsthm}

\usepackage{wrapfig}

\usepackage[capitalize,noabbrev]{cleveref}

\theoremstyle{plain}
\newtheorem{theorem}{Theorem}[section]

\newtheorem{lemma}[theorem]{Lemma}

\theoremstyle{definition}

\theoremstyle{remark}

\newcommand{\bx}{\mathbf{x}}
\usepackage[textsize=tiny]{todonotes}

\title{Alignment of Diffusion Model and Flow Matching \\ for Text-to-Image Generation}

\author[1]{Yidong~Ouyang}
\author[2]{Liyan~Xie}
\author[3]{Hongyuan~Zha}
\author[1]{Guang~Cheng}

\affil[1]{\small Department of Statistics,
  University of California, Los Angeles
}

\affil[2]{\small
Department of Industrial and Systems Engineering, University of Minneasota}

\affil[3]{\small School of Data Science, The Chinese University of Hong Kong, Shenzhen}

\date{}

\begin{document}

\maketitle

\vspace{-0.2in}

\begin{abstract}
Diffusion models and flow matching have demonstrated remarkable success in text-to-image generation. While many existing alignment methods primarily focus on fine-tuning pre-trained generative models to maximize a given reward function, these approaches require extensive computational resources and may not generalize well across different objectives. In this work, we propose a novel alignment framework by leveraging the underlying nature of the alignment problem---sampling from reward-weighted distributions---and show that it applies to both diffusion models (via score guidance) and flow matching models (via velocity guidance). The score function (velocity field) required for the reward-weighted distribution can be decomposed into the pre-trained score (velocity field) plus a conditional expectation of the reward. For the alignment on the diffusion model, we identify a fundamental challenge: the adversarial nature of the guidance term can introduce undesirable artifacts in the generated images. Therefore, we propose a finetuning-free framework that trains a guidance network to estimate the conditional expectation of the reward. We achieve comparable performance to finetuning-based models with one-step generation with at least a 60\% reduction in computational cost. For the alignment on flow matching, we propose a training-free framework that improves the generation quality without additional computational cost.



\end{abstract}

\section{Introduction}
\label{sec:intro}
Diffusion models and flow matching have achieved impressive performance in text-to-image generation, as demonstrated by state-of-the-art models such as Imagen \citep{Saharia2022PhotorealisticTD}, DALL-E 3 \citep{BetkerImprovingIG}, and Stable Diffusion \citep{Rombach2021HighResolutionIS}. These models have been proven capable of generating high-quality, creative images even from novel and complex text prompts.

Inspired by Reinforcement Learning from Human Feedback (RLHF) \citep{Ouyang2022TrainingLM}, many alignment approaches leverage preference pairs to fine-tune models for generating samples that align with task-specific objectives. RLHF-type methods typically learn a reward function and use the policy gradients to update the model \citep{Lee2023AligningTM, Fan2023DPOKRL, Black2023TrainingDM, Clark2023DirectlyFD,MaxMin-RLHF,Jaques2016SequenceTC, Liu2025FlowGRPOTF, Jaques2020HumancentricDT}. On the other hand, Direct Preference Optimization (DPO)-type methods directly optimize the model to adhere to human preferences, without requiring explicit reward modeling or reinforcement learning \citep{Rafailov2023DirectPO, Wallace2023DiffusionMA, Yang2023UsingHF, Liang2024StepawarePO, Yang2024ADR}.

Despite their effectiveness, these approaches require modifying model parameters through fine-tuning, which comes with several potential limitations. For example, fine-tuning for new reward functions is computationally expensive and often requires carefully designed training strategies; otherwise, optimizing on a limited set of input prompts can limit generalization to unseen prompts. More importantly, existing fine-tuning approaches do not fully exploit the structure of the alignment problem. Instead, they typically apply Low-Rank Adaptation (LoRA) to optimize model weights for a specific reward function \cite{Hu2021LoRALA}, which may not be the most efficient strategy.


In contrast, plug-and-play alignment methods integrate new objectives without modifying the underlying model parameters, significantly reducing computational costs while adapting flexibly to different reward functions. In this paper, we cast alignment for both diffusion models and flow matching models as a unified sampling problem from reward-weighted distributions. Under this formulation, the key object needed for sampling---the new score function for diffusion or the new velocity field for flow matching---can be written as the corresponding pre-trained quantity plus an additional reward-driven guidance term.

For diffusion models, the guidance term admits an adversarial nature flaw, i.e., the guidance is the gradient of the log conditional expectation of the reward. Directly using the gradient of high dimensional input space can lead to undesirable artifacts in the generated images. To address this issue, we propose a finetuning-free alignment method that trains a lightweight guidance network to estimate the required conditional expectation, together with a regularization strategy that stabilizes the guidance landscape. We evaluate the effectiveness of the proposed method on four widely used criteria for text-to-image generation, and the proposed method achieves comparable performance to finetuning-based models in one-step generation while reducing computational cost by at least 60\%.

For flow matching, we derive the exact form of velocity guidance and further propose a training-free estimator that directly computes the guidance term without additional model fine-tuning. The proposed method improves the generation quality without
additional training overhead.

\begin{figure*}[t!]
\centering
  \includegraphics[width=0.95\textwidth]{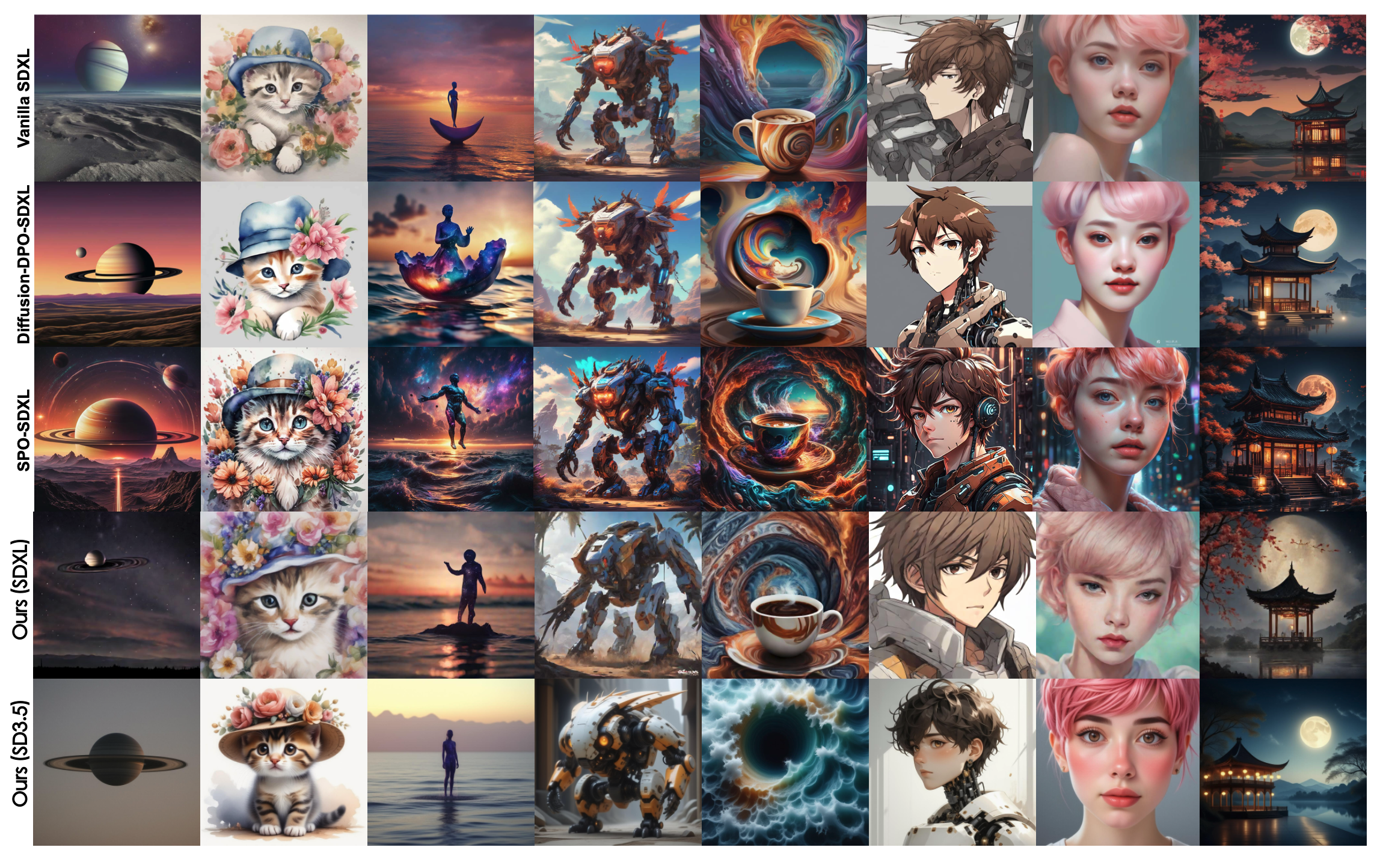}
  \caption{Qualitative comparison with Vanilla SDXL, Diffusion-DPO, and SPO. Our method achieves better aesthetic quality and stronger alignment with the text prompt. Prompts are provided in the Appendix \ref{ap:prompt}.}
  \label{fig:f1}
\end{figure*}



\section{Preliminaries}
\label{sec:preli}
In this section, we begin with a brief overview of diffusion models and flow matching in Section \ref{sec:latant} and Section \ref{sec:flow}. We then review existing techniques for aligning pre-trained models with human preferences, and decompose the alignment procedure into two key components: reward learning for modeling human preferences in Section \ref{sec:reward} and the alignment methods in Section \ref{sec:alignment}.

\subsection{Diffusion Models}\label{sec:latant}

Diffusion generative models are characterized by their forward and backward processes \cite{Ho2020DDPM, Song2021ScoreBasedGM}. The forward process gradually injects Gaussian noise into samples $\mathbf{x}_0$ from the data distribution following the stochastic differential equation:
\begin{equation}\label{eq:forward}
    \mathrm{d} \mathbf{x}_t=\mathbf{f}(\mathbf{x}_t, t) \mathrm{d} t+g(t) \mathrm{d} \mathbf{w}, \ t\in[0,T],
\end{equation}
where $\mathbf{w}$ is the standard Brownian motion, $\mathbf{f}(\cdot, t): \mathbb{R}^d \rightarrow \mathbb{R}^d$ is a drift coefficient, and $g(\cdot): \mathbb{R} \rightarrow \mathbb{R}$ is a  diffusion coefficient. We use $p_t(\mathbf{x})$ to denote the marginal distribution of $\mathbf{x}_t$ at time $t$. And we can use the time reversal of \eqref{eq:forward} for generation, which admits the following form \citep{Anderson1982ReversetimeDE}:
\begin{equation}\label{eq:backward}
\mathrm{d} \mathbf{x}_t=\left[\mathbf{f}(\mathbf{x}_t, t)-g(t)^2 \nabla_{\mathbf{x}} \log p_t(\mathbf{x})\right] \mathrm{d} t+g(t) \mathrm{d} \overline{\mathbf{w}},
\end{equation}
where $\overline{\mathbf{w}}$ is a standard Brownian motion when time flows backwards from $T$ to 0, and $\mathrm{d}t$ is an infinitesimal negative time step. The score function of each marginal distribution $\nabla_{\mathbf{x}} \log p_t(\mathbf{x})$ needs to be estimated by the following score matching objective: \begin{equation}\label{eq:dsm}
 \underset{\boldsymbol{\theta}}{\min} \ \mathbb{E}_t\left\{\lambda(t) \mathbb{E}_{p_t(\mathbf{x}_t)} \left[\left\|\mathbf{s}_{\boldsymbol{\theta}}(\mathbf{x}_t, t)-
\nabla_{\mathbf{x}_t} \log p_t(\mathbf{x}_t)\right\|_2^2\right]\right\},
\end{equation}
where $\lambda(t):[0, T] \rightarrow \mathbb{R}_{>0}$ is a positive weighting function, $t$ is uniformly sampled over $[0, T]$. The latent diffusion model \citep{Rombach2021HighResolutionIS, Podell2023SDXLIL} further extends diffusion models to text-to-image generation. They use an image encoder $\mathcal{E}$ that maps $\mathbf{x}$ into a latent representation and use a text encoder $\tau$ that maps the prompts $y$ into an embedding as the condition.


\subsection{Flow Matching}\label{sec:flow}

Flow matching models learn a time-dependent velocity field that transports a simple base distribution to the data distribution \citep{Lipman2022FlowMF} via the probability flow ODE
\begin{equation}\label{eq:fm_ode}
\frac{\mathrm{d}\mathbf{x}_t}{\mathrm{d}t} = \mathbf{v}_{\boldsymbol{\phi}}(\mathbf{x}_t, y, t), \quad t \in [0,1], \notag
\end{equation}
where $\mathbf{v}_{\boldsymbol{\phi}}:\mathbb{R}^d \to \mathbb{R}^d$ is a learnable velocity field. Unlike diffusion models, we denote $\mathbf{x}_0$ as a sample from a base distribution (e.g., standard Gaussian) and $\mathbf{x}_1$ as a sample from the data distribution.

The flow matching objective minimizes the discrepancy between the model vector field and the oracle velocity field along the trajectory:
\begin{equation}
\mathcal{L}(\theta) = \mathbb{E}_{\mathbf{x}_t \sim p_t(\mathbf{x}_t|\mathbf{x}_0,\mathbf{x}_1), t \sim \mathcal{U}[0,1]}
\Big[ \big\| \mathbf{v}_{\boldsymbol{\phi}}(\mathbf{x}_t, t) - \mathbf{v}(\mathbf{x}_t, y, t) \big\|^2_2 \Big],
\end{equation}
where $\mathbf{x}_t$ is a linear interpolation between $\mathbf{x}_0$ and $\mathbf{x}_1$, 
and $\mathbf{v}(\mathbf{x}_t, y, t)$ is the oracle velocity field.

\subsection{Reward Learning}\label{sec:reward}

The Bradley-Terry (BT) model \citep{Bradley1952RankAO}, and the more general Plackett-Luce ranking models \citep{Plackett1975TheAO, Luce1979IndividualCB}, are commonly used to model preferences. Given a prompt $y$ and a pair of responses $\mathbf{x}_w \succ \mathbf{x}_l \mid y$, where $\mathbf{x}_w$ denotes the winning response and $\mathbf{x}_l$ denotes the losing response under the preference of humans. The BT model depicts the preference distribution as $$
p\left(\mathbf{x}_w \succ \mathbf{x}_l \mid y\right)=\frac{\exp \left(r\left(\mathbf{x}_w, y\right)\right)}{\exp \left(r\left(\mathbf{x}_w, y\right)\right)+\exp \left(r\left(\mathbf{x}_l, y\right)\right)},
$$
where $r(\mathbf{x}, y)$ denotes the reward model and can be learned by the following maximum likelihood objective, 
\begin{equation}
\label{eq:learn_reward}
\underset{\boldsymbol{\phi}}{\min} -\mathbb{E}_{\left(\mathbf{x}_w, \mathbf{x}_l, y\right) \sim \mathcal{D}}\left[\log \sigma\left(r\left(\mathbf{x}_w, y\right)-r\left(\mathbf{x}_l, y\right)\right)\right],
\end{equation}
where $\mathcal{D}=\{\mathbf{x}_w^{(i)}, \mathbf{x}_l^{(i)}, y^{(i)}\}_{i=1}^N$ is the offline preference dataset and $\sigma$ denotes the logistic function.


\subsection{Alignment}\label{sec:alignment}
Building on the success of alignment techniques for finetuning large pre-trained models, many studies have explored aligning diffusion models and flow matching with human preferences. We review these approaches in the following.
\vspace{-2pt}
\paragraph{Reinforcement Learning from Human Feedback.}
This type of works \citep{Lee2023AligningTM, Xu2023ImageRewardLA, Fan2023DPOKRL, Black2023TrainingDM, Clark2023DirectlyFD} finetune the pre-trained model $\pi_{\text{ref}}$ by policy gradient objectives \citep{Jaques2016SequenceTC, Jaques2020HumancentricDT}. In particular, the fine-tuned model $\pi_\theta$ is obtained by solving the following optimization problem:
\begin{equation}\label{eq:max_under_kl}
\max _{\pi_\theta} \  \mathbb{E}_{y \sim \mathcal{D}_{\text{prompt}}, \mathbf{x} \sim \pi_\theta(\mathbf{x} \mid y)}\left[r(\mathbf{x}, y)\right] 
-\beta \mathbb{D}_{\mathrm{KL}}\left[\pi_\theta(\mathbf{x} \mid y) \| \pi_{\text{ref}}(\mathbf{x} \mid y)\right],    
\end{equation}
where $\mathcal{D}_{\text{prompt}}$ denotes the prompt dataset. This type of method requires a pre-trained reward function for policy optimization \citep{Schulman2017ProximalPO}.

\paragraph{Direct Preference Optimization.}
Rafailov et al. \citep{Rafailov2023DirectPO} propose not to explicitly learn the reward function. They start with the analytic solution of \eqref{eq:max_under_kl} as the energy-guided form,
\begin{equation}
\label{eq:energy_weighted}
\pi_{\text{r}}(\mathbf{x} \mid y)=\frac{1}{Z(y)} \pi_{\mathrm{ref}}(\mathbf{x} \mid y) \exp \left(\frac{1}{\beta} r(\mathbf{x}, y)\right),
\end{equation}
where $Z(y)=\int \pi_{\text {ref }}(\mathbf{x} \mid y) \exp \left(\frac{1}{\beta} r(\mathbf{x}, y)\right) \mathrm{d} \mathbf{x}$ is the partition function. Therefore, they can reparameterize the reward function $r(\mathbf{x}, y)$ as 
\begin{equation}
\label{eq:reward}
r(\mathbf{x}, y)=\beta \log \frac{\pi_{\text{r}}(\mathbf{x} \mid y)}{\pi_{\text {ref }}(\mathbf{x} \mid y)}+\beta \log Z(y).
\end{equation}
Plugging \eqref{eq:reward} into \eqref{eq:learn_reward} yields the objective of DPO-type methods:
\begin{equation}\label{eq:DPO}
\min -\mathbb{E}_{\left(\mathbf{x}_w, \mathbf{x}_l, y\right) \sim \mathcal{D}} \left[\log \sigma\left(\beta \log \frac{\pi_\theta\left(\mathbf{x}_w \mid y\right)}{\pi_{\mathrm{ref}}\left(\mathbf{x}_w \mid y\right)}-\beta \log \frac{\pi_\theta\left(\mathbf{x}_l \mid y\right)}{\pi_{\mathrm{ref}}\left(\mathbf{x}_l \mid y\right)}\right)\right].
\end{equation}

\section{Proposed Framework for Diffusion Model}
\label{sec:method}
In this section, we introduce a finetuning-free frameworks to directly sample from the reward-guided distribution for diffusion models. We begin by introducing the methodology formulation in Section \ref{sec:guidance}. We then provide an in-depth analysis of several vanilla methods for calculating the guidance in Section \ref{sec:vanilla}. We highlight that these vanilla guidance methods exhibit adversarial guidance, which generates undesirable artifacts and worsens performance, particularly in text-to-image generation. Then, we present an enhanced method in Section \ref{sec:FFG} that alleviates the problem. 

\subsection{Methodology Formulation}\label{sec:guidance}
Inspired by previous works from transfer learning \citep{Ouyang2024TransferLF}, we consider preference learning in terms of transferring a pre-trained diffusion model to adapt to the given preference data. To this end, we propose a finetuning-free alignment method for the diffusion models. Instead of using RLHF-type (like \eqref{eq:max_under_kl}) or DPO-type (like \eqref{eq:DPO}) alignments, we propose to directly sample from the reward-weighted distribution $\pi_{\text{r}}(\mathbf{x}|y)$ in \eqref{eq:energy_weighted} leveraging the relationships between score functions in the following Theorem. 

\begin{figure*}[ht!]
\centering
  \includegraphics[width=0.7\textwidth]{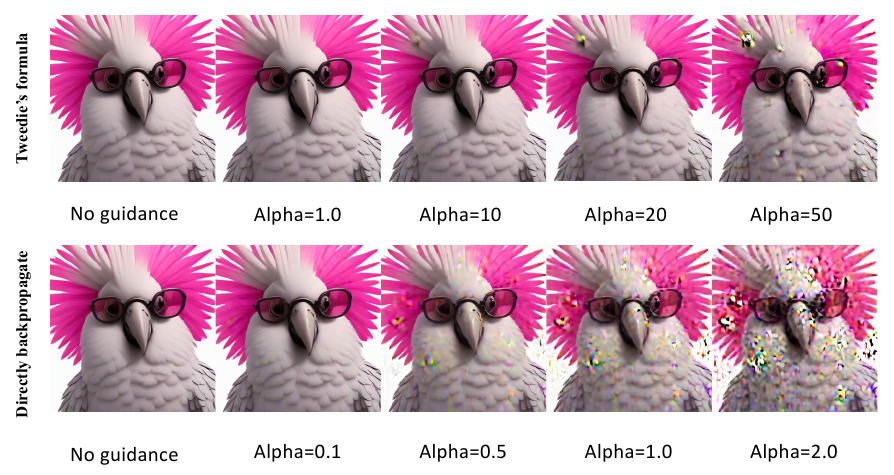}
  \caption{Illustration of the Adversarial Nature of Guidance. When the strength of the guidance is too small, there is little difference between the generated images with or without guidance. However, as the magnitude of the guidance increases (from left to right), undesirable artifacts become more pronounced. The prompt is "A 3D Rendering of a cockatoo wearing sunglasses. The sunglasses have a deep black frame with bright pink lenses. Fashion photography, volumetric lighting, CG rendering".}
  \label{fig:adv}
\end{figure*}

\begin{theorem}\label{thm:IS_guidance}
Let the conditional distribution of reference diffusion model $\pi_{\text{ref}}(\bx|y)$ be denoted as distribution $p$ and the reward-weighted distribution $\pi_{\text{r}}(\bx|y)$ defined in \eqref{eq:energy_weighted} as distribution $q$. Under some mild assumption of the forward noising process detailed in Appendix \ref{app:proofs}, let $\boldsymbol{\phi}^*$ be the optimal solution for the conditional diffusion model trained on target domain $q(\mathbf{x}_0, y)$, i.e.,
\begin{equation}
\boldsymbol{\phi}^* = \underset{\boldsymbol{\phi}}{\arg \min } ~\mathbb{E}_t\Big\{ \lambda(t) 
\mathbb{E}_{q_t(\mathbf{x}_t, y)} \Big[ \Big\| 
\mathbf{s}_{\boldsymbol{\phi}}(\mathbf{x}_t, y, t) - \nabla_{\mathbf{x}_t} \log q_t(\mathbf{x}_t | y)
\Big\|_2^2 \Big] \Big\},     \notag
\end{equation}
then 
\begin{equation}\label{eq:reward_guided}
\mathbf{s}_{\boldsymbol{\phi}^*}(\mathbf{x}_t, y, t) =\underbrace{\nabla_{\mathbf{x}_t} \log p_t(\mathbf{x}_t |  y)}_{\substack{\text {pre-trained conditional model}\\ \text{on source}}} + \underbrace{\nabla_{\mathbf{x}_t} \log \mathbb{E}_{p(\mathbf{x}_0|\mathbf{x}_t, y)}\left[ \exp( \frac{1}{\beta}r(\mathbf{x}_0, y))\right]}_{\text {conditional guidance}}.     
\end{equation}
\end{theorem}
The proof can be found in Appendix \ref{app:proofs}. Based on \eqref{eq:reward_guided}, we can calculate the additional guidance term rather than finetuning the text-to-image generative model. In general, the guidance term in \eqref{eq:reward_guided} is not straightforward to compute as we need to sample from $p(\mathbf{x}_0|\mathbf{x}_t, y)$ for each $\mathbf{x}_t$ in the generation process. In the following, we first discuss some existing ways to calculate the guidance term.

\subsection{Vanilla Method to Compute the Guidance Term} \label{sec:vanilla}



\paragraph{M1: Direct backpropagate through diffusion process.}

The first method directly backpropagates through diffusion process to calculate $\nabla_{\mathbf{x}_t} \log \mathbb{E}_{p(\mathbf{x}_0|\mathbf{x}_t, y)}[ \exp( r(\mathbf{x}_0, y)/\beta)]$ for fine-tuning the diffusion model. In \citep{Song2023LossGuidedDM}, the authors propose an unbiased Monte Carlo estimation:
\[
    \nabla_{\mathbf{x}_t} \log \mathbb{E}_{p(\mathbf{x}_0|\mathbf{x}_t, y)}
    \Big[ \exp\Big( \frac{1}{\beta} r(\mathbf{x}_0, y) \Big) \Big]  \approx \nabla_{\mathbf{x}_t} \log 
    \frac{1}{n} \sum_{i=1}^{n} \exp\Big( \frac{1}{\beta} r(\mathbf{x}_0^i, y) \Big),
    \]
where $\mathbf{x}_0^i$ denotes the $i$-th sample drawn from $p(\mathbf{x}_0|\mathbf{x}_t, y)$. However, this Monte Carlo estimation significantly increases memory costs, especially in text-to-image generation.
Inspired by recent studies \citep{Clark2023DirectlyFD}, we can borrow the same techniques, e.g., accumulated gradients along the diffusion process using techniques such as low-rank adaptation (LoRA) \citep{Hu2021LoRALA} and truncation or gradient checkpointing \citep{Prabhudesai2023AligningTD, Clark2023DirectlyFD}, to alleviate the memory cost of backpropagating through the diffusion process for calculating the guidance term. We can further reduce the memory cost by using the few-step diffusion model as the reference model. Despite these techniques, the memory requirements remain higher compared to the proposed approach.

\paragraph{M2: Approximate and apply Tweedie's formula.}

The second method first approximates the guidance term by \cite{Chung2022DiffusionPS}:
\begin{equation}\label{eq:biased}
     \nabla_{\mathbf{x}_t} \log \mathbb{E}_{p(\mathbf{x}_0|\mathbf{x}_t, y)}
    \big[ \exp\big( \frac{1}{\beta}r(\mathbf{x}_0, y) \big) \big]  \approx \frac{1}{\beta} \nabla_{\mathbf{x}_t} r\big(    \mathbb{E}_{p(\mathbf{x}_0|\mathbf{x}_t, y)}[\mathbf{x}_0], y \big).    
\end{equation}
Then, Tweedie's formula is further applied by \cite{Bansal2023UniversalGF, Chung2022DiffusionPS, Yu2023FreeDoMTE}:
\[
\mathbb{E}\left[\boldsymbol{x}_0 \mid \boldsymbol{x}_t, y\right]=\boldsymbol{x}_t + \sigma_t^2 \nabla_{\boldsymbol{x}_t} \log p_{t}\left(\boldsymbol{x}_t|y\right).
\]

However, as noted in \citep{Lu20, Song2023LossGuidedDM}, the approximation used in \eqref{eq:biased} is biased, leading to an incorrect calculation of the guidance term.

In the following, we empirically evaluate the effectiveness of these methods for aligning text-to-image generation tasks. We first identify a previously overlooked issue that contributes to suboptimal alignment performance.
%
%
%
%
%
Figure~\ref{fig:adv} illustrates the performance of two vanilla methods under the guidance of PickScore \citep{Kirstain2023PickaPicAO}, a reward function that evaluates whether the generated images align with human aesthetic and semantic preferences. The x-axis represents the strength of the guidance term, denoted by $\alpha$ \footnote{Although there is no $\alpha$ in \eqref{eq:reward_guided}, many guidance methods \citep{Lu20, Song2023LossGuidedDM} add this hyperparameter in practice to balance the strength of the guidance term with the score.}. Our experiments reveal that tuning this hyperparameter presents significant challenges. Insufficient values of $\alpha$
produce results indistinguishable from unguided generation, while excessive values introduce substantial artifacts that degrade image quality.

We attribute this phenomenon to the adversarial nature of the guidance mechanism, as observed in prior work \citep{Shen2024UnderstandingAI}. In \eqref{eq:reward_guided}, the guidance term is directly added to the estimated score. If the landscape is not smooth or does not behave well\footnote{We use landscape to describe the change of reward given the change of images.}, the adversarial nature of the guidance can lead to undesirable artifacts in the generated images. 
To address these limitations, our proposed framework provides theoretical guarantees for generating properly aligned distributions with a fixed strength parameter $\alpha = 1$. Furthermore, we develop an additional regularization technique for training the guidance network that mitigates these instability issues.


\subsection{Proposed Finetuning-free Guidance for Diffison Models} \label{sec:FFG}

We first utilize the following trick to calculate the conditional expectation, similar to previous works \citep{Ouyang2024TransferLF, Lu20}. 
\begin{lemma}\label{thm:exact_guidance}
For a neural network $h_{\boldsymbol{\psi}}\left(\mathbf{x}_t, y, t\right)$ parameterized by $\boldsymbol{\psi}$, define the objective  
\begin{equation} \label{eq:guidance}
\mathcal{L}_{\text{guidance}}(\boldsymbol{\psi}) :=\mathbb{E}_{p(\mathbf{x}_0, \mathbf{x}_t, y)}\left[\left\|h_{\boldsymbol{\psi}}\left(\mathbf{x}_t, y, t\right)-\exp(\frac{1}{\beta}r(\mathbf{x}_0, y))\right\|_2^2\right],    
\end{equation}
then its minimizer $\boldsymbol{\psi}^* = \underset{\boldsymbol{\psi}}{\arg \min } \ \mathcal{L}_{\text{guidance}}(\boldsymbol{\psi})$ satisfies:
\[
h_{\boldsymbol{\psi}^*}\left(\mathbf{x}_t, y, t\right)=\mathbb{E}_{p(\mathbf{x}_0 |\mathbf{x}_t, y)}\left[{\exp(\frac{1}{\beta}r(\mathbf{x}_0, y))}\right].
\]
\end{lemma}
By Lemma \ref{thm:exact_guidance}, we can instead estimate the value $\mathbb{E}_{p(\mathbf{x}_0 |\mathbf{x}_t, y)}[\exp(r(\mathbf{x}_0, y)/\beta)]$ using the guidance network $h_{\boldsymbol{\psi}^*}$ obtained by minimizing the objective function $\mathcal{L}_{\text{guidance}}(\boldsymbol{\psi})$, which can be approximated by easy sampling from the joint distribution $p(\mathbf{x}_0, \mathbf{x}_t, y)$. Then, the estimated score function for the aligned diffusion model can be calculated as follows:
\begin{equation}\label{eq:dsm_IS_sampling}
\mathbf{s}_{\boldsymbol{\phi}^*}(\mathbf{x}_t, y, t) = \underbrace{\nabla_{\mathbf{x}_t} \log p(\mathbf{x}_t | y)}_{\substack{\text {pre-trained model}\\ \text{on source}}}+ \underbrace{\nabla_{\mathbf{x}_t} \log h_{\boldsymbol{\psi}^*}\left(\mathbf{x}_t, y, t\right)}_{\text {guidance network}}.
\end{equation}

\begin{table*}[t]
\centering
\caption{Comparison of finetuning-free alignment algorithms on diffusion models. Our method uniquely provides theoretical guarantees for the correct form for guidance with a step size guarantee.}
\label{tab:comparison}
\resizebox{\linewidth}{!}{%
\begin{tabular}{lcccc}
\toprule
\textbf{Method} & \textbf{Classifier Guidance} & \textbf{Direct backpropagate (M1)} & \textbf{Tweedie’s formula (M2)} & \textbf{Ours} \\
\midrule
Formulation & $\frac{1}{\beta} \nabla_{\mathbf{x}_t} r\big(\mathbf{x}_t, y\big)$ & 
$\nabla_{\mathbf{x}_t} \log \frac{1}{n} \sum_{i=1}^{n} \exp\Big( \frac{1}{\beta} r(\mathbf{x}_0^i, y) \Big)$ & 
$\frac{1}{\beta} \nabla_{\mathbf{x}_t} r\big( \mathbb{E}_{p(\mathbf{x}_0|\mathbf{x}_t, y)}[\mathbf{x}_0], y \big)$ & 
$\nabla_{\mathbf{x}_t} \log h_{\boldsymbol{\psi}^*}\left(\mathbf{x}_t, y, t\right)$ \\
\midrule
Unbiased & \xmark & \cmark & \xmark & \cmark \\
Step size guarantee & \xmark & \xmark & \xmark & \cmark \\
\bottomrule
\end{tabular}}
\end{table*}

To alleviate the adversarial nature of the guidance, we can adopt the consistency regularization $\mathcal{L}_{\text{consistence}}$ to learn the guidance network $h_{\boldsymbol{\psi}^*}$ better, i.e., the gradient of $\mathcal{L}_{\text{consistence}}\left(\mathbf{x}_t, y, t\right)$ with respect to $\mathbf{x}_t$ should match the score in preferred data. The key point of this regularization is that we cannot easily change the landscape of a given predetermined reward function, but we can regularize the landscape of the learned guidance network to ensure the generation of high-quality images.
\begin{equation}\label{eq:consistence}
\begin{aligned}
\boldsymbol{\psi}^* 
&= \underset{\boldsymbol{\psi}}{\arg \min } \ \mathcal{L}_{\text{consistence}} \\
&:= \mathbb{E}_{q(\mathbf{x}_0, y)} 
\mathbb{E}_{q(\mathbf{x}_t|\mathbf{x}_0)} \Big[
\big\|\nabla_{\mathbf{x}_t} \log p(\mathbf{x}_t|\mathbf{x}_0, y)  + \nabla_{\mathbf{x}_t} \log h_{\boldsymbol{\psi}}\left(\mathbf{x}_t, y, t\right)
- \nabla_{\mathbf{x}_t} \log q(\mathbf{x}_t | \mathbf{x}_0, y) 
\big\|_2^2 \Big].
\end{aligned}
\end{equation}

Combining the consistency regularization terms together with the original guidance loss in \eqref{eq:guidance}, the final learning objective for the guidance network can be described as follows:
\begin{equation}\label{eq:final_obj}
\boldsymbol{\psi}^* = \underset{\boldsymbol{\psi}}{\arg \min } \ \{\mathcal{L}_{\text{guidance}}+ \eta \ \mathcal{L}_{\text{consistence}} \},  
\end{equation}
where $\eta \geq 0$ are hyperparameters that control the strength of additional regularization, which also enhances the flexibility of our solution scheme.

\subsection{Further Improvement to One-step Generation}
The training objectives in ~\eqref{eq:guidance} and \eqref{eq:consistence} are agnostic to the reference model, indicating that we can use any pre-trained diffusion model with any reward function, whether differentiable or not. Motivated by the computational efficiency of one-step generative models in practical applications, we further present a straightforward approach for applying our proposed finetuning-free guidance to one-step text-to-image models.

Specifically, instead of sampling \( t \) uniformly from \([0,T]\), we can simply set \( t = T \). This small modification offers several advantages. First, while one-step diffusion models may not perform as well as few-step (2–4 step) models~\citep{salimans2022progressive}, we empirically find that with additional guidance, their performance improves significantly, as presented in Section~\ref{sec:abl}. Second, as the guidance network \( h_{\boldsymbol{\psi}} \) now becomes time-independent, we empirically observe that \( h_{\boldsymbol{\psi}} \) is easy to train—with ten training epochs on the Pick-a-Pic V1 dataset, our guidance network produces high-quality images, which can be found in Section \ref{sec:exp_res}. We summarize the overall learning pipeline in Algorithm \ref{alg:guidance_training_regularized} in the Appendix. And we leave another two gradient-free designs for diffusion models in Appendix \ref{apen:gradient-free}.

\section{Proposed Framework for Flow Matching}
\label{sec:method_flow}
\paragraph{Training-free Alignment Framework for Flow Matching}

Given that state-of-the-art models are grounded in Diffusion Transformers \citep{Peebles2022ScalableDM} and flow matching \citep{Lipman2022FlowMF}, we present the exact form of flow-matching guidance in the theorem below.

\begin{theorem}\label{thm:fm}
Let $\boldsymbol{\phi}_q^*$ be the optimal solution for the conditional flow matching model trained on target domain $q(\mathbf{x}_1, y)$ (where $\mathbf{x}_1$ are sampled from data distribution, $\mathbf{v}_q(\mathbf{x}_t, y, t)$ denotes the oracle velocity field on target distribution), i.e., $\boldsymbol{\phi}_q^*$ equals
\begin{align*}
 \underset{\boldsymbol{\phi}}{\arg \min }\mathbb{E}_t \Big\{
\mathbb{E}_{q_t(\mathbf{x}_t, y)} \Big[ \Big\| 
\mathbf{v}_{\boldsymbol{\phi}}(\mathbf{x}_t, y, t) - \mathbf{v}_q(\mathbf{x}_t, y, t)
\Big\|_2^2 \Big] \Big\}, 
\end{align*}
then 
\begin{equation}\label{eq:reward_guided_fm}
\mathbf{v}_{\boldsymbol{\phi}^*_q}(\mathbf{x}_t, y, t) 
= \mathbf{v}_{\boldsymbol{\phi}_p}(\mathbf{x}_t, y, t)+  \mathbb{E}_{\mathbf{x}_1 \sim p_{1 \mid t}(\mathbf{x}_1 \mid \mathbf{x}_t, y)}\bigg[ \big( R(\bx_1,\bx_t,y) - 1 \big) \boldsymbol{v}_t(\mathbf{x}_t \mid \mathbf{x}_1, y) \bigg],
\end{equation}
where 
\[
R(\bx_1,\bx_t,y)=\frac{\exp \left(\frac{1}{\beta} r(\mathbf{x}_1, y)\right)} {\mathbb{E}_{\mathbf{x}_1' \sim p_{1 \mid t}(\mathbf{x}_1 \mid \mathbf{x}_t, y)} \left[\exp \left(\frac{1}{\beta} r(\mathbf{x}_1', y)\right)\right]}.
\]
\end{theorem}

\paragraph{Estimation of the Guidance Term for Flow Matching}
Different from the guidance term of diffusion models in \eqref{eq:reward_guided}, the guidance for flow matching in \eqref{eq:reward_guided_fm} does not have the adversarial problem. The guidance term is a conditional expectation without the gradient operator. 

To enable fast sampling, we would like to use importance sampling to convert the conditional expectation under $p(\mathbf{x}_1 \mid \mathbf{x}_t, y)$ into an expectation under $p(\mathbf{x}_1\mid y)$. After the detailed derivation in \ref{ap:proof_fm}, we can calculate the guidance term of flow matching by \begin{align*} \E_{\mathbf{x}_1 \sim p(\mathbf{x}_1 \mid y)}
\Bigg[
\Bigg(
\frac{
\exp\!\big(\frac{1}{\beta} r(\mathbf{x}_1, y)\big)
}{
\mathbb{E}_{\mathbf{x}_1} \Big[\exp\!\big(\frac{1}{\beta} r(\mathbf{x}_1, y)\big)
\frac{p_{t|1}(\mathbf{x}_t \mid \mathbf{x}_1, y)}{\mathbb{E}_{\mathbf{x}_1}[p_{t|1}(\mathbf{x}_t \mid \mathbf{x}_1, y)]} \Big]
} - 1
\Bigg) \boldsymbol{v}_t(\mathbf{x}_t \mid \mathbf{x}_1, y) \frac{p_{t|1}(\mathbf{x}_t \mid \mathbf{x}_1, y)}
{\mathbb{E}_{\mathbf{x}_1}[p_{t|1}(\mathbf{x}_t \mid \mathbf{x}_1, y)]}
\Bigg].
\end{align*}

Therefore, we do not need to sample $\mathbf{x}_1$ with multiple function evaluations, but just sample from the marginal data distribution. Compared with the finetuning-free method proposed in \eqref{eq:dsm_IS_sampling}, this formulation is training-free and offers greater computational efficiency.

\section{Experimental Results}
\label{sec:exp}
In this section, we present a comprehensive experimental evaluation, demonstrating the effectiveness of our two frameworks for sampling directly from reward-guided distributions. We first outline our experimental setup and evaluation criteria in Section \ref{sec:exp_setup}, followed by benchmark results against state-of-the-art methods in Section \ref{sec:exp_res}. Finally, we provide an in-depth ablation study that validates our key theoretical claims and demonstrates the superior performance of our guidance network in Section \ref{sec:abl}.

\begin{table*}[t!]
\centering
\caption{Benchmark comparison of different methods on text-to-image alignment. Results are grouped by base model.}
\resizebox{\textwidth}{!}{
\begin{tabular}{c|c|c|c|c|c|c}
\toprule
Type & Method & PickScore & HPSV2 & ImageReward & Aesthetic & Training GPU Hour \\
\midrule

\multicolumn{7}{c}{\textbf{Base Model: SDXL}} \\
\midrule
Baseline & SDXL & 21.95 & 26.95 & 0.5380 & 5.950 & -- \\
\midrule
\multirow{2}{*}{Training-free}
 & Direct backpropagate & 21.84 & 27.53 & 0.5870 & 5.922 & -- \\
 & Tweedie's formula & 22.34 & 28.76 & 0.9501 & 6.002 & -- \\
\midrule
\multirow{2}{*}{Finetuning-based}
 & Diff.-DPO & 22.64 & 29.31 & 0.9436 & 6.015 & 4800 \\
 & SPO & 23.06 & 31.80 & \textbf{1.0803} & 6.364 & 234 \\
\midrule
Finetuning-free
 & Ours & \textbf{23.08} & \textbf{32.12} & 1.0625 & \textbf{6.452} & \textbf{92} \\

\midrule\midrule
\multicolumn{7}{c}{\textbf{Base Model: SD3.5 Large Turbo}} \\
\midrule
Baseline
 & SD3.5 Large Turbo & 22.30 & 30.29 & 1.0159 & 6.5190 & -- \\
\midrule
Finetuning-free
 & Ours & \textbf{23.14} & \textbf{32.31} & \textbf{1.1025} & \textbf{6.5280} & -- \\
\bottomrule
\end{tabular}}
\label{tab:pick}
\end{table*}

\subsection{Experimental Setup}\label{sec:exp_setup}
For the experiments on diffusion models, we follow the official configurations recommended for SPO \citep{Liang2024StepawarePO}, Diffusion-DPO \citep{Wallace2023DiffusionMA}, and MAPO \citep{She2024MAPOAM}. Diffusion-DPO and MAPO are fine-tuned on the Pick-a-Pic V2 dataset, which contains over 800k image preference pairs. In contrast, SPO is fine-tuned online using 4k text prompts (without images) randomly selected from Pick-a-Pic V1. Our method trains the guidance network offline using 583k image preference pairs from Pick-a-Pic V1. Overall, our method and the competing models in the text-to-image alignment benchmark are trained on comparable datasets, allowing for a fair comparison. We adopt Stable Diffusion XL (SDXL)-Turbo \citep{Sauer2023AdversarialDD} as the reference model for one-step text-to-image generation. For the experiments on flow matching, we adopt the state-of-the-art SD3.5 Large Turbo \citep{Esser2024ScalingRF} as the backbone. The official recommendation for the number of sampling steps is four to eight, and we use four steps for all experiments.

\paragraph{Implementation Details.}
In the following, we provide the training details for the guidance network of the diffusion model. Since the guidance network takes noisy images $\mathbf{x}_T$ and prompts $y$ as input and outputs a scalar value, we adopt the same variational autoencoder (VAE), tokenizer, and text encoder from the reference diffusion model for encoding image and text. Consequently, the trainable parameters of our guidance network are quite small. In practice, we adopt two convolutional layers for processing VAE-encoded feature maps and a five-layer multi-layer perceptron (MLP) to project the image and text embedding to a scalar. The total parameter size of the guidance network is only 72 MB, making it lightweight and easy to train. We train the guidance network on the Pick-a-Pic training dataset for 10 epochs with batch size 32, Adam optimizer, learning rate 1e-3, and hyperparameters $\eta=1$.

\paragraph{Evaluation Criterion.}
Following established evaluation protocols \citep{Wallace2023DiffusionMA, Liang2024StepawarePO}, we report quantitative results using 500 validation prompts from the validation unique split of Pick-a-Pic. We adopt four evaluation criteria to evaluate different aspects of image quality. PickScore \citep{Kirstain2023PickaPicAO} measures overall human preference by aggregating judgments on aesthetic appeal, coherence, and realism. HPSV2 \citep{Wu2023HumanPS} assesses prompt adherence, ensuring the generated image accurately reflects the given textual description. ImageReward \citep{Xu2023ImageRewardLA} quantifies human preference based on fine-grained attributes such as composition, detail preservation, and semantic relevance. Lastly, the aesthetic evaluation model from LAION \citep{schuhmann2022laion} focuses on visual appeal, capturing factors such as color harmony, style, and artistic quality.

\subsection{Experimental Results} \label{sec:exp_res}
As shown in Table \ref{tab:pick}, our method surpasses baseline approaches across four evaluation criteria, demonstrating the effectiveness of the two proposed frameworks in enhancing text-to-image alignment. The improvements are observed in both perceptual quality and semantic coherence, indicating that our guidance network successfully refines image generation to better match textual descriptions. This performance gain highlights the advantages of our lightweight architecture and the optimization strategy used during training. Figure \ref{fig:f1} provides a qualitative comparison with baseline methods, further illustrating the superior visual fidelity and text alignment achieved by our approach.

\subsection{Ablation study} \label{sec:abl}
In this section, we first verify the advantages of our proposed method against other finetuning-free guidance methods as summarized in Table \ref{tab:comparison}. We then analyze the impact of few-step (2–4 step) generation compared to one-step generation, highlighting how our guidance term significantly enhances performance.

As illustrated in Figure \ref{fig:apple}, vanilla guidance methods struggle to induce meaningful improvements in generated images, even with carefully tuned guidance strength. Increasing the guidance parameter $\alpha$ often leads to undesirable artifacts rather than quality improvements. In contrast, our method effectively enhances image generation by leveraging a regularized guidance network, demonstrating its ability to refine scene details and improve alignment with input prompts.

To further explore this, we examine the performance of our method against two vanilla guidance techniques, Tweedie’s and Backpropagate, as well as the no guidance baseline, all under a one-step sampling condition. As shown in Table \ref{tab:ablation}, our method achieves the highest PickScore. This demonstrates that our regularized guidance network provides a substantial improvement over no guidance scenario and traditional methods. Consistent with prior studies, increasing the number of steps from one to two or three results in improved image quality, as shown in Figure \ref{fig:apple} and Table \ref{tab:ablation}. However, our method enables one-step generation to achieve performance even better than 2- or 3-step generation, highlighting the power of our guidance network. In Appendix \ref{ap:ablation}, we include the sensitive analysis of the regularization strength.

\begin{figure*}[t!]
\centering
  \includegraphics[width=1\textwidth]{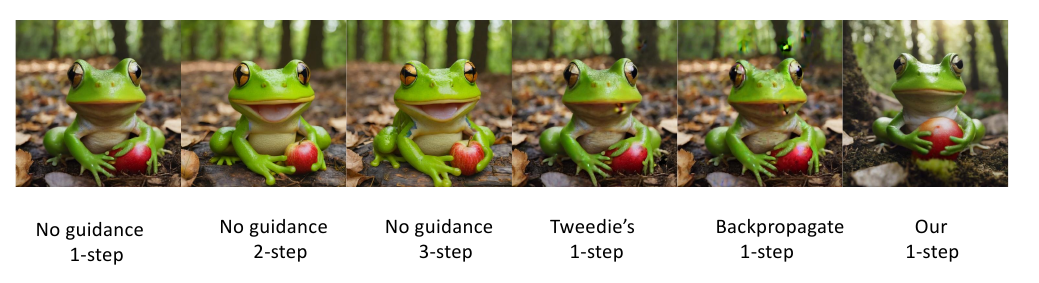}
  \caption{Effectiveness of the proposed method for diffusion models: The results demonstrate that 2-step and 3-step generation significantly improve the quality of the generated images compared to one-step generation. While two vanilla guidance methods (Tweedie’s formula or directly backpropagation summarized in Section \ref{sec:vanilla}) fail to produce meaningful changes in the scene despite appropriate guidance strength, our method successfully achieves this enhancement. The prompt is ``A photo of a frog holding an apple while smiling in the forest''.}
  \label{fig:apple}
\end{figure*}
\begin{table}[t!]
    \centering
    \caption{Ablation study comparing the performance of our method with no guidance and two vanilla guidance methods under one-step and multi-step generation. Our method outperforms all baselines, which demonstrates the effectiveness of our guidance network in refining image quality and prompt alignment.}
    \label{tab:ablation}
    \begin{tabular}{c|c}
\toprule  
Method & PickScore \\
\midrule
Ours (1 step) & \textbf{23.08} \\
No guidance (1 step) & 22.14 \\
Tweedie's (1 step) & 22.34 \\
Backpropagate (1 step) & 21.84 \\
No guidance (2 steps) & 22.64 \\
No guidance (3 steps) & 22.56 \\
\bottomrule
\end{tabular}
\end{table}




\section{Related Work}\label{sec:ap_related}
Existing alignment methods can be broadly categorized into two approaches: RLHF-based method that uses policy gradient to update the diffusion models and flow matching, and DPO-based methods that use a parametrization trick to update the diffusion models without explicitly learning the reward function.

\paragraph{RLHF-based alignment of diffusion model and flow matching.}
Lee et al. \citep{Lee2023AligningTM} first train a reward model to predict human feedback and adopt a reward-weighted finetuning objective to align the diffusion model. In \citep{Fan2023DPOKRL, Black2023TrainingDM}, diffusion models are updated using policy gradient algorithms under Kullback–Leibler (KL) constraints. Clark et al. \citep{Clark2023DirectlyFD} propagate gradients of the reward function through the full sampling procedure, and reduce memory costs by adopting low-rank adaptation (LoRA) \citep{Hu2021LoRALA} and gradient checkpointing \citep{Chen2016TrainingDN}. In \cite{Liu2025FlowGRPOTF, Li2025MixGRPOUF, Xue2025DanceGRPOUG, He2025TempFlowGRPOWT}, the authors improve GRPO \citep{Shao2024DeepSeekMathPT} for the alignment of flow matching.

\paragraph{DPO-based alignment of diffusion model.}
A line of work \citep{Wallace2023DiffusionMA, Yang2023UsingHF} directly applies DPO \citep{Rafailov2023DirectPO} to align the diffusion model with human preference. Liang et al. \citep{Liang2024StepawarePO} propose a step-aware preference model and a step-wise resampler to align the preference optimization target with the denoising performance at each timestep. Yang et al. \citep{Yang2024ADR} take on a finer dense reward perspective and derive a tractable alignment objective that emphasizes the initial steps.  

\paragraph{Training-free guidance.}

This line of work \citep{Chung2022DiffusionPS, Graikos2022DiffusionMA, Lu20, Song2023LossGuidedDM, Bansal2023UniversalGF, Yu2023FreeDoMTE, Shen2024UnderstandingAI, Ye2024TFGUT} explores the use of diffusion models as plug-and-play priors for solving inverse problems. Some work \cite{Shen2024UnderstandingAI, Tang2024InferenceTimeAO, Uehara2024UnderstandingRL, Ma2025InferenceTimeSF, Singhal2025AGF} study inference-time optimization for alignment. However, to the best of our knowledge, there has been limited exploration of applying guidance on diffusion models to address the challenge of text-to-image alignment in the context of one-step generation. Also, there has been limited exploration of training-free guidance on flow matching of text-to-image alignment. This gap motivates our work.

\section{Conclusion}\label{sec:conclusion}
In this paper, we introduced two novel framework for aligning text-to-image diffusion models and flow matching models with human preferences. By formulating alignment as sampling from a reward-weighted distribution, our approach leverages a plug-and-play guidance mechanism. Specifically, we decomposed the score function (velocity field) of the reward-weighted distribution into the pre-trained score (velocity field) plus a reward-driven guidance term. For diffusion models, we identify that the adversarial nature of the guidance can introduce undesirable artifacts, and we propose a finetuning-free approach that trains a lightweight guidance network to estimate the conditional expectation of the reward, together with a regularization strategy that stabilizes the guidance landscape. Empirically, our method achieves performance comparable to finetuning-based approaches for one-step generation while reducing computational cost by at least 60\%. For flow matching, we derive the exact form of velocity guidance and propose a training-free estimator that improves generation quality without additional training.


\bibliography{references}
\bibliographystyle{plain}

\appendix
\onecolumn
\clearpage

\section{Theoretical Details for Section \ref{sec:method}} \label{app:proofs}

\subsection{Proof of Theorem \ref{thm:IS_guidance}}\label{ap:proof_thm_ISguidance}
We first restate the complete theorem as follows:
\begin{theorem}\label{thm:IS_guidance_app}
Let the conditional distribution of reference diffusion model $\pi_{\text{ref}}(\bx|y)$ be denoted as distribution $p$ and the reward-weighted distribution $\pi_{\text{r}}(\bx|y)$ defined in \eqref{eq:energy_weighted} as distribution $q$. Assume $\mathbf{x}_t$ and $y$ are conditionally independent given $\mathbf{x}_0$ in the forward process, i.e.,  $p(\mathbf{x}_t|\mathbf{x}_0,y)=p(\mathbf{x}_t|\mathbf{x}_0)$, $\forall t\in[0,T]$. 
Additionally, assume the forward process on the reward-weighted distribution is identical to that on the reference distribution $q(\mathbf{x}_t | \mathbf{x}_0)=p(\mathbf{x}_t | \mathbf{x}_0)$\footnote{These two assumptions are mild since $\mathbf{x}_0$ contains all information about $y$ and $p(\mathbf{x}_t | \mathbf{x}_0)$ and $q(\mathbf{x}_t | \mathbf{x}_0)$ are forward noising process, which is easy to control.},
and $\boldsymbol{\phi}^*$ is the optimal solution for the conditional diffusion model trained on target domain $q(\mathbf{x}_0, y)$, i.e.,
\begin{equation}
\boldsymbol{\phi}^* = \underset{\boldsymbol{\phi}}{\arg \min } ~\mathbb{E}_t\Big\{\lambda(t) 
\mathbb{E}_{q_t(\mathbf{x}_t, y)} \Big[ \Big\| 
\mathbf{s}_{\boldsymbol{\phi}}(\mathbf{x}_t, y, t) - \nabla_{\mathbf{x}_t} \log q_t(\mathbf{x}_t | y)
\Big\|_2^2 \Big] \Big\},
\end{equation}
then 
\begin{equation}\label{eq:reward_guided_app}
\mathbf{s}_{\boldsymbol{\phi}^*}(\mathbf{x}_t, y, t) =\underbrace{\nabla_{\mathbf{x}_t} \log p_t(\mathbf{x}_t |  y)}_{\substack{\text {pre-trained conditional model}\\ \text{on source}}}+ \underbrace{\nabla_{\mathbf{x}_t} \log \mathbb{E}_{p(\mathbf{x}_0|\mathbf{x}_t, y)}\left[ \exp( \frac{1}{\beta}r(\mathbf{x}_0, y))\right]}_{\text {conditional guidance}}. 
\end{equation}
\end{theorem}
\begin{proof}
The proof is based on the theoretical framework of \cite{Ouyang2024TransferLF}. For the ease of readers, we incorporate the relevant conclusion from their work as
lemmas below. To prove \eqref{eq:reward_guided_app}, we first build the connection between the Conditional Score Matching on the target domain and Importance Weighted Conditional Denoising Score Matching on the source domain in the following Lemma:

\begin{lemma}
\label{DSM-equ-guidance}
Conditional Score Matching on the target domain is equivalent to Importance Weighted Denoising Score Matching on the source domain, i.e.,  
\[
\begin{aligned}
\boldsymbol{\phi}^*= &\underset{\boldsymbol{\phi}}{\arg \min } \ \mathbb{E}_t\left\{\lambda(t) \mathbb{E}_{q_t(\mathbf{x}_t, y)} \left[\left\|\mathbf{s}_{\boldsymbol{\phi}}(\mathbf{x}_t, y, t)-
\nabla_{\mathbf{x}_t} \log q_t(\mathbf{x}_t | y)\right\|_2^2\right]\right\}\\
= & \underset{\boldsymbol{\phi}}{\arg \min } \ \mathbb{E}_t\left\{\lambda(t) \mathbb{E}_{p(\mathbf{x}_0, y)} \mathbb{E}_{p(\mathbf{x}_t|\mathbf{x}_0)}\left[\left\|\mathbf{s}_{\boldsymbol{\phi}}(\mathbf{x}_t, y,  t)-
\nabla_{\mathbf{x}_t} \log p(\mathbf{x}_t | \mathbf{x}_0)\right\|_2^2 \frac{q(\mathbf{x}_0, y)}{p(\mathbf{x}_0, y)}\right]\right\}.
\end{aligned}
\]
\end{lemma}

\begin{proof}[Proof of Lemma \ref{DSM-equ-guidance}]
We first connect the Conditional Score Matching objective in the target domain to the Conditional Denoising Score Matching objective in target distribution, which is proven by \cite{Batzolis2021ConditionalIG}, i.e.,
\[
\begin{aligned}
\boldsymbol{\phi}^*= &\underset{\boldsymbol{\phi}}{\arg \min } \ \mathbb{E}_t\left\{\lambda(t) \mathbb{E}_{q_t(\mathbf{x}_t, y)} \left[\left\|\mathbf{s}_{\boldsymbol{\phi}}(\mathbf{x}_t, y, t)-
\nabla_{\mathbf{x}_t} \log q_t(\mathbf{x}_t|y)\right\|_2^2\right]\right\}\\= &\underset{\boldsymbol{\phi}}{\arg \min } \ \mathbb{E}_t\left\{\lambda(t) \mathbb{E}_{q(\mathbf{x}_0,y)} \mathbb{E}_{q(\mathbf{x}_t|\mathbf{x}_0)}\left[\left\|\mathbf{s}_{\boldsymbol{\phi}}(\mathbf{x}_t, y, t)-
\nabla_{\mathbf{x}_t} \log q(\mathbf{x}_t | \mathbf{x}_0)\right\|_2^2\right]\right\}.
\end{aligned}
\]

Then we split the mean squared error of the Conditional Denoising Score Matching objective on the target distribution into three terms as follows:
\begin{align}\label{eq:DSM_target}
&\mathbb{E}_{q(\mathbf{x}_0, y)} \mathbb{E}_{q(\mathbf{x}_t|\mathbf{x}_0)}\left[\left\|\mathbf{s}_{\boldsymbol{\phi}}(\mathbf{x}_t, y, t)-
\nabla_{\mathbf{x}_t} \log q(\mathbf{x}_t | \mathbf{x}_0)\right\|_2^2\right]\notag\\
=& \mathbb{E}_{q(\mathbf{x}_0, \mathbf{x}_t, y)}\left[\left\|\mathbf{s}_{\boldsymbol{\phi}}(\mathbf{x}_t, y, t)\right\|_2^2\right] - 2\mathbb{E}_{q(\mathbf{x}_0,\mathbf{x}_t, y)}\left[\langle \mathbf{s}_{\boldsymbol{\phi}}(\mathbf{x}_t, y, t), \nabla_{\mathbf{x}_t} \log q(\mathbf{x}_t | \mathbf{x}_0) \rangle\right]  + C_1,
\end{align}
where $C_1=\mathbb{E}_{q(\mathbf{x}_0,\mathbf{x}_t, y)}\left[\left\| \nabla_{\mathbf{x}_t} \log q(\mathbf{x}_t | \mathbf{x}_0)\right\|_2^2\right]$ is a constant independent with $\boldsymbol{\phi}$, and $q(\mathbf{x}_t|\mathbf{x}_0,y)=q(\mathbf{x}_t|\mathbf{x}_0) $ because of conditional independent of $\mathbf{x}_t$ and $y$
given $\mathbf{x}_0$ by assumption. We can similarly split the mean squared error of Denoising Score Matching on the source domain into three terms as follows:
\begin{equation}\label{eq:ISDSM_source}
\begin{aligned}
& \mathbb{E}_{p(\mathbf{x}_0, y)} \mathbb{E}_{p(\mathbf{x}_t|\mathbf{x}_0)}\left[\left\|\mathbf{s}_{\boldsymbol{\phi}}(\mathbf{x}_t, y, t)-
\nabla_{\mathbf{x}_t} \log p(\mathbf{x}_t | \mathbf{x}_0)\right\|_2^2 \frac{q(\mathbf{x}_0, y)}{p(\mathbf{x}_0, y)}\right]\\
=& \mathbb{E}_{p(\mathbf{x}_0, \mathbf{x}_t, y)}\left[\left\|\mathbf{s}_{\boldsymbol{\phi}}(\mathbf{x}_t, y, t)\right\|_2^2\frac{q(\mathbf{x}_0, y)}{p(\mathbf{x}_0, y)}\right] - 2\mathbb{E}_{p(\mathbf{x}_0,\mathbf{x}_t, y)}\left[\langle \mathbf{s}_{\boldsymbol{\phi}}(\mathbf{x}_t, y, t), \nabla_{\mathbf{x}_t} \log p(\mathbf{x}_t | \mathbf{x}_0) \rangle\frac{q(\mathbf{x}_0, y)}{p(\mathbf{x}_0, y)}\right] \\
&+ C_2,
\end{aligned}
\end{equation}
where $C_2$ is a constant independent with $\boldsymbol{\phi}$.

It is obvious to show that the first term in \eqref{eq:DSM_target} is equal to the first term in \eqref{eq:ISDSM_source}, i.e., 
\[
\begin{aligned}
&\mathbb{E}_{p(\mathbf{x}_0, \mathbf{x}_t, y)}\left[\left\|\mathbf{s}_{\boldsymbol{\phi}}(\mathbf{x}_t, y, t)\right\|_2^2\frac{q(\mathbf{x}_0, y)}{p(\mathbf{x}_0, y)}\right]\\
= &\int_{\mathbf{x}_0} \int_{\mathbf{x}_t} \int_y p(\mathbf{x}_0, y) p(\mathbf{x}_t| \mathbf{x}_0) \left\|\mathbf{s}_{\boldsymbol{\phi}}(\mathbf{x}_t, y, t)\right\|_2^2\frac{q(\mathbf{x}_0, y)}{p(\mathbf{x}_0, y)}d \mathbf{x}_0 d \mathbf{x}_t d y\\
= &\int_{\mathbf{x}_0} \int_{\mathbf{x}_t} \int_y p(\mathbf{x}_0, y) q(\mathbf{x}_t| \mathbf{x}_0) \left\|\mathbf{s}_{\boldsymbol{\phi}}(\mathbf{x}_t, y, t)\right\|_2^2\frac{q(\mathbf{x}_0, y)}{p(\mathbf{x}_0, y)}d \mathbf{x}_0 d \mathbf{x}_t d y\\
= &\int_{\mathbf{x}_0} \int_{\mathbf{x}_t} \int_{y} q(\mathbf{x}_0,\mathbf{x}_t, y) \left\|\mathbf{s}_{\boldsymbol{\phi}}(\mathbf{x}_t, y, t)\right\|_2^2d \mathbf{x}_0 d \mathbf{x}_t d y\\
= &\mathbb{E}_{q(\mathbf{x}_0, \mathbf{x}_t, y)}\left[\left\|\mathbf{s}_{\boldsymbol{\phi}}(\mathbf{x}_t, y, t)\right\|_2^2\right].
\end{aligned}
\]
And the second term is also equivalent:
\[
\begin{aligned}
&\mathbb{E}_{p(\mathbf{x}_0,\mathbf{x}_t, y)}\left[\langle \mathbf{s}_{\boldsymbol{\phi}}(\mathbf{x}_t, y, t), \nabla_{\mathbf{x}_t} \log p(\mathbf{x}_t | \mathbf{x}_0) \rangle\frac{q(\mathbf{x}_0, y)}{p(\mathbf{x}_0, y)}\right]\\
=& \int_{\mathbf{x}_0} \int_{\mathbf{x}_t} \int_y p(\mathbf{x}_0, \mathbf{x}_t, y) \langle \mathbf{s}_{\boldsymbol{\phi}}(\mathbf{x}_t, y, t),  \frac{\nabla_{\mathbf{x}_t} p(\mathbf{x}_t | \mathbf{x}_0)}{p(\mathbf{x}_t | \mathbf{x}_0)}  \rangle\frac{q(\mathbf{x}_0, y)}{p(\mathbf{x}_0, y)}d \mathbf{x}_0 d \mathbf{x}_t d y\\
=& \int_{\mathbf{x}_0} \int_{\mathbf{x}_t} \int_y p(\mathbf{x}_0, \mathbf{x}_t, y) \langle \mathbf{s}_{\boldsymbol{\phi}}(\mathbf{x}_t, y, t),  \frac{\nabla_{\mathbf{x}_t} q(\mathbf{x}_t | \mathbf{x}_0)}{p(\mathbf{x}_t | \mathbf{x}_0)}  \rangle\frac{q(\mathbf{x}_0, y)}{p(\mathbf{x}_0, y)}d \mathbf{x}_0 d \mathbf{x}_t d y\\
=& \int_{\mathbf{x}_0} \int_{\mathbf{x}_t} \int_y \langle \mathbf{s}_{\boldsymbol{\phi}}(\mathbf{x}_t, y, t),  \nabla_{\mathbf{x}_t} q(\mathbf{x}_t | \mathbf{x}_0)  \rangle q(\mathbf{x}_0, y)d \mathbf{x}_0 d \mathbf{x}_t d y\\
=& \int_{\mathbf{x}_0} \int_{\mathbf{x}_t} \int_y \langle \mathbf{s}_{\boldsymbol{\phi}}(\mathbf{x}_t, y, t),  \nabla_{\mathbf{x}_t} \log q(\mathbf{x}_t | \mathbf{x}_0)  \rangle q(\mathbf{x}_t | \mathbf{x}_0) q(\mathbf{x}_0, y)d \mathbf{x}_0 d \mathbf{x}_t d y\\
=&\mathbb{E}_{q(\mathbf{x}_0,\mathbf{x}_t, y)}\left[\langle \mathbf{s}_{\boldsymbol{\phi}}(\mathbf{x}_t, y, t), \nabla_{\mathbf{x}_t} \log q(\mathbf{x}_t | \mathbf{x}_0) \rangle\right].
\end{aligned}
\]
\end{proof}

\begin{lemma}\label{thm:IS_guidance_conditional}
    Assume $\mathbf{x}_t$ and $y$ are conditional independent given $\mathbf{x}_0$ in the forward process, i.e.,  $p(\mathbf{x}_t|\mathbf{x}_0,y)=p(\mathbf{x}_t|\mathbf{x}_0)$, $\forall t\in[0,T]$, and let the forward process on the target domain be identical to that on the source domain $q(\mathbf{x}_t | \mathbf{x}_0)=p(\mathbf{x}_t | \mathbf{x}_0)$, and $\boldsymbol{\phi}^*$ is the optimal solution for the conditional diffusion model trained on target domain $q(\mathbf{x}_0, y)$, i.e.,
    \begin{equation}
\boldsymbol{\phi}^*= \underset{\boldsymbol{\phi}}{\arg \min } ~\mathbb{E}_t\left\{\lambda(t) \mathbb{E}_{q_t(\mathbf{x}_t,y)} \left[\left\|\mathbf{s}_{\boldsymbol{\phi}}(\mathbf{x}_t, y, t)-
\nabla_{\mathbf{x}_t} \log q_t(\mathbf{x}_t | y)\right\|_2^2\right]\right\},
\end{equation}
then 
\begin{equation}\label{eq:guide_conditional}
\begin{aligned}
\mathbf{s}_{\boldsymbol{\phi}^*}(\mathbf{x}_t, y, t) = \nabla_{\mathbf{x}_t} \log p_t(\mathbf{x}_t |  y)+ \nabla_{\mathbf{x}_t} \log \mathbb{E}_{p(\mathbf{x}_0|\mathbf{x}_t, y)}\left[\frac{q(\mathbf{x}_0, y)}{p(\mathbf{x}_0, y)}\right].
\end{aligned}    
\end{equation}
\end{lemma}

\begin{proof}[Proof of Lemma \ref{thm:IS_guidance_conditional}]
According to Lemma \ref{DSM-equ-guidance}, the optimal solution satisfies
\[
\begin{aligned}
\boldsymbol{\phi}^*&= \underset{\boldsymbol{\phi}}{\arg \min } ~\mathbb{E}_t\left\{\lambda(t) \mathbb{E}_{p(\mathbf{x}_0, y)} \mathbb{E}_{p(\mathbf{x}_t|\mathbf{x}_0)}\left[\left\|\mathbf{s}_{\boldsymbol{\phi}}(\mathbf{x}_t, y,  t)-
\nabla_{\mathbf{x}_t} \log p(\mathbf{x}_t | \mathbf{x}_0)\right\|_2^2 \frac{q(\mathbf{x}_0, y)}{p(\mathbf{x}_0, y)}\right]\right\}\\
\end{aligned}
\]
where $Z(y)=\int p(\mathbf{x}_0, y) \exp \left(\frac{1}{\beta} r(\mathbf{x}_0, y)\right) \mathrm{d} \mathbf{x}$.
Then, we use Importance Weighted Conditional Denoising Score Matching on the source domain to get the analytic form of $\mathbf{s}_{\boldsymbol{\phi}^*}$ as follows: 
\[
\mathbf{s}_{\boldsymbol{\phi}^*}(\mathbf{x}_t, y, t) = \frac{\mathbb{E}_{p(\mathbf{x}_0|\mathbf{x}_t, y)}\left[\nabla_{\mathbf{x}_t} \log p(\mathbf{x}_t | \mathbf{x}_0)\frac{q(\mathbf{x}_0, y)}{p(\mathbf{x}_0, y)}   \right]}{\mathbb{E}_{p(\mathbf{x}_0|\mathbf{x}_t, y)}\left[\frac{q(\mathbf{x}_0, y)}{p(\mathbf{x}_0, y)}   \right]}.
\]

Moreover, the RHS of \eqref{eq:guide_conditional} can be rewritten as:
\[
\begin{aligned}
\text{RHS} =& \nabla_{\mathbf{x}_t} \log p_t(\mathbf{x}_t|y)+ \nabla_{\mathbf{x}_t} \log \mathbb{E}_{p(\mathbf{x}_0|\mathbf{x}_t, y)}\left[\frac{q(\mathbf{x}_0, y)}{p(\mathbf{x}_0, y)}\right]\\
=& \nabla_{\mathbf{x}_t} \log p_t(\mathbf{x}_t|y)+ \frac{\nabla_{\mathbf{x}_t}  \mathbb{E}_{p(\mathbf{x}_0|\mathbf{x}_t, y)}\left[\frac{q(\mathbf{x}_0, y)}{p(\mathbf{x}_0, y)}\right]}{\mathbb{E}_{p(\mathbf{x}_0|\mathbf{x}_t, y)}\left[\frac{q(\mathbf{x}_0, y)}{p(\mathbf{x}_0, y)}\right]}\\
=& \nabla_{\mathbf{x}_t} \log p_t(\mathbf{x}_t|y)+ \frac{  \mathbb{E}_{p(\mathbf{x}_0|\mathbf{x}_t, y)}\left[\frac{q(\mathbf{x}_0, y)}{p(\mathbf{x}_0, y)}\nabla_{\mathbf{x}_t} \log p(\mathbf{x}_0|\mathbf{x}_t, y)\right]}{\mathbb{E}_{p(\mathbf{x}_0|\mathbf{x}_t, y)}\left[\frac{q(\mathbf{x}_0, y)}{p(\mathbf{x}_0, y)}\right]}.
\end{aligned}
\]
Since 
\[
\begin{aligned} 
\nabla_{\mathbf{x}_t} \log p(\mathbf{x}_0|\mathbf{x}_t, y)&=\nabla_{\mathbf{x}_t} \log p(\mathbf{x}_t|\mathbf{x}_0, y)+ \nabla_{\mathbf{x}_t} \log p(\mathbf{x}_0|y) - \nabla_{\mathbf{x}_t} \log p_t(\mathbf{x}_t|y)\\
&= \nabla_{\mathbf{x}_t} \log p(\mathbf{x}_ t|\mathbf{x}_0, y) - \nabla_{\mathbf{x}_t} \log p_t(\mathbf{x}_t|y),\\
&= \nabla_{\mathbf{x}_t} \log p(\mathbf{x}_ t|\mathbf{x}_0) - \nabla_{\mathbf{x}_t} \log p_t(\mathbf{x}_t|y),
\end{aligned}
\]
we can further simplify the RHS of \eqref{eq:guide_conditional} as follows:
\[
\begin{aligned}
\text{RHS} =& \nabla_{\mathbf{x}_t} \log p_t(\mathbf{x}_t|y)+ \frac{  \mathbb{E}_{p(\mathbf{x}_0|\mathbf{x}_t, y)}\left[\frac{q(\mathbf{x}_0, y)}{p(\mathbf{x}_0, y)}\nabla_{\mathbf{x}_t} \log p(\mathbf{x}_t|\mathbf{x}_0)\right]}{\mathbb{E}_{p(\mathbf{x}_0|\mathbf{x}_t, y)}\left[\frac{q(\mathbf{x}_0, y)}{p(\mathbf{x}_0, y)}\right]} -\nabla_{\mathbf{x}_t} \log p_t(\mathbf{x}_t|y)\\
=&\frac{\mathbb{E}_{p(\mathbf{x}_0|\mathbf{x}_t, y)}\left[\nabla_{\mathbf{x}_t} \log p(\mathbf{x}_t | \mathbf{x}_0)\frac{q(\mathbf{x}_0, y)}{p(\mathbf{x}_0, y)}   \right]}{\mathbb{E}_{p(\mathbf{x}_0|\mathbf{x}_t, y)}\left[\frac{q(\mathbf{x}_0, y)}{p(\mathbf{x}_0, y)}   \right]}\\
=& \mathbf{s}_{\boldsymbol{\phi}^*}(\mathbf{x}_t, t).
\end{aligned}
\]
Thereby, we finish the proof.
\end{proof}

According to the lemma \ref{thm:IS_guidance_conditional}, we replace the density ratio $\frac{q(\mathbf{x}_0, y)}{p(\mathbf{x}_0, y)}$ by $\frac{\exp \left(\frac{1}{\beta} r(\mathbf{x}_0, y)\right)}{Z(y)}$, we get 
\[
\begin{aligned}
\mathbf{s}_{\boldsymbol{\phi}^*}(\mathbf{x}_t, y, t) &= \nabla_{\mathbf{x}_t} \log p_t(\mathbf{x}_t |  y)+ \nabla_{\mathbf{x}_t} \log \mathbb{E}_{p(\mathbf{x}_0|\mathbf{x}_t, y)}\left[\frac{q(\mathbf{x}_0, y)}{p(\mathbf{x}_0, y)}\right]\\
&=\nabla_{\mathbf{x}_t} \log p_t(\mathbf{x}_t |  y)+ \nabla_{\mathbf{x}_t} \log \mathbb{E}_{p(\mathbf{x}_0|\mathbf{x}_t, y)}\left[\frac{\exp \left(\frac{1}{\beta} r(\mathbf{x}_0, y)\right)}{Z(y)}\right]\\
&=\nabla_{\mathbf{x}_t} \log p_t(\mathbf{x}_t |  y)+ \nabla_{\mathbf{x}_t} \log \mathbb{E}_{p(\mathbf{x}_0|\mathbf{x}_t, y)}\left[\exp \left(\frac{1}{\beta} r(\mathbf{x}_0, y)\right)\right]
\end{aligned}    
\]
Thereby, we finish the proof.
\end{proof}

\subsection{Proof of Theorem \ref{thm:fm}}\label{ap:proof_fm}
We provide a detailed discussion about training-free guidance of flow matching in this subsection.
\begin{proof}[Proof of Theorem \ref{thm:fm}]
Denote $\boldsymbol{v}_t(\mathbf{x}_t, y)$ and $\boldsymbol{v}_t(\mathbf{x}_t \mid \mathbf{x}_1, y)$ 
as the marginal and conditional velocities, respectively. Then we have
\[
\begin{aligned}
    \boldsymbol{v}_t^q(\mathbf{x}_t, y)
    &= \mathbb{E}_{\mathbf{x}_1 \sim q_{1 \mid t}(\mathbf{x}_1 \mid \mathbf{x}_t, y)}
        \left[ \boldsymbol{v}_t(\mathbf{x}_t \mid \mathbf{x}_1, y) \right] \\[4pt]
    &= \mathbb{E}_{\mathbf{x}_1 \sim p_{1 \mid t}(\mathbf{x}_1 \mid \mathbf{x}_t, y)}
        \left[
            \boldsymbol{v}_t(\mathbf{x}_t \mid \mathbf{x}_1, y)
            \frac{q_{1 \mid t}(\mathbf{x}_1 \mid \mathbf{x}_t, y)}{p_{1 \mid t}(\mathbf{x}_1 \mid \mathbf{x}_t, y)}
        \right] \\[4pt]
        &=
\mathbb{E}_{\mathbf{x}_1 \sim p_{1|t}(\mathbf{x}_1 \mid \mathbf{x}_t, y)}
\left[
\boldsymbol{v}_t(\mathbf{x}_t \mid \mathbf{x}_1, y)\,
\frac{
\frac{q_{t|1}(\mathbf{x}_t \mid \mathbf{x}_1,y)\, q_1(\mathbf{x}_1)}
     {q_t(\mathbf{x}_t,y)}
}{
\frac{p_{t|1}(\mathbf{x}_t \mid \mathbf{x}_1,y)\, p_1(\mathbf{x}_1)}
     {p_t(\mathbf{x}_t,y)}
}
\right]
\\[6pt]
&=
\mathbb{E}_{\mathbf{x}_1 \sim p_{1|t}(\mathbf{x}_1 \mid \mathbf{x}_t, y)}
\left[
\boldsymbol{v}_t(\mathbf{x}_t \mid \mathbf{x}_1, y)\,
\frac{
q_{t|1}(\mathbf{x}_t \mid \mathbf{x}_1,y)\, q_1(\mathbf{x}_1)\, p_t(\mathbf{x}_t,y)
}{
p_{t|1}(\mathbf{x}_t \mid \mathbf{x}_1,y)\, p_1(\mathbf{x}_1)\, q_t(\mathbf{x}_t,y)
}
\right]
\\[6pt]
&=
\mathbb{E}_{\mathbf{x}_1 \sim p_{1|t}(\mathbf{x}_1 \mid \mathbf{x}_t, y)}
\left[
\boldsymbol{v}_t(\mathbf{x}_t \mid \mathbf{x}_1, y)\,
\frac{
q_1(\mathbf{x}_1)
}{
p_1(\mathbf{x}_1)
}
\cdot
\frac{
p_t(\mathbf{x}_t,y)
}{
q_t(\mathbf{x}_t,y)
}
\right]
\qquad(\text{because } q_{t|1}(\mathbf{x}_t \mid \mathbf{x}_1, y)=p_{t|1}(\mathbf{x}_t \mid \mathbf{x}_1, y))
\\[6pt]
&=
\mathbb{E}_{\mathbf{x}_1 \sim p_{1|t}(\mathbf{x}_1 \mid \mathbf{x}_t, y)}
\left[
\boldsymbol{v}_t(\mathbf{x}_t \mid \mathbf{x}_1, y)\,
\frac{
\frac{q_1(\mathbf{x}_1)}{p_1(\mathbf{x}_1)}
}{
\frac{q_t(\mathbf{x}_t,y)}{p_t(\mathbf{x}_t,y)}
}
\right]
\\[6pt]
&=
\mathbb{E}_{\mathbf{x}_1 \sim p_{1|t}(\mathbf{x}_1 \mid \mathbf{x}_t, y)}
\left[
\boldsymbol{v}_t(\mathbf{x}_t \mid \mathbf{x}_1, y)\,
\frac{
\frac{q_1(\mathbf{x}_1)}{p_1(\mathbf{x}_1)}
}{
\sum_{\mathbf{x}_1'} 
p_{1|t}(\mathbf{x}_1' \mid \mathbf{x}_t,y)
\frac{q_1(\mathbf{x}_1')}{p_1(\mathbf{x}_1')}
}
\right]
\\[6pt]
&=
\mathbb{E}_{\mathbf{x}_1 \sim p_{1|t}(\mathbf{x}_1 \mid \mathbf{x}_t, y)}
\left[
\boldsymbol{v}_t(\mathbf{x}_t \mid \mathbf{x}_1, y)\,
\frac{
\frac{q_1(\mathbf{x}_1)}{p_1(\mathbf{x}_1)}
}{
\mathbb{E}_{\mathbf{x}_1' \sim p_{1|t}(\mathbf{x}_1 \mid \mathbf{x}_t, y)}
\left[
\frac{q_1(\mathbf{x}_1')}{p_1(\mathbf{x}_1')}
\right]
}
\right]
\\[6pt]
&=
\mathbb{E}_{\mathbf{x}_1 \sim p_{1|t}(\mathbf{x}_1 \mid \mathbf{x}_t, y)}
\left[
\boldsymbol{v}_t(\mathbf{x}_t \mid \mathbf{x}_1, y)\,
\frac{
\exp\!\left(\frac{1}{\beta} r(\mathbf{x}_1,y)\right)
}{
\mathbb{E}_{\mathbf{x}_1' \sim p_{1|t}(\mathbf{x}_1 \mid \mathbf{x}_t, y)}
\left[
\exp\!\left(\frac{1}{\beta} r(\mathbf{x}_1',y)\right)
\right]
}
\right]\\
    &= \mathbb{E}_{\mathbf{x}_1 \sim p_{1 \mid t}(\mathbf{x}_1 \mid \mathbf{x}_t, y)}
        \left[
            \boldsymbol{v}_t(\mathbf{x}_t \mid \mathbf{x}_1, y)
            \frac{\exp \left(\frac{1}{\beta} r(\mathbf{x}_1, y)\right)}
                {\mathbb{E}_{\mathbf{x}_1' \sim p_{1 \mid t}(\mathbf{x}_1 \mid \mathbf{x}_t, y)}
                    \left[\,\exp \left(\frac{1}{\beta} r(\mathbf{x}_1', y)\right)\,\right]}
        \right] \\[4pt]
            &= \boldsymbol{v}_t^p(\mathbf{x}_t, y)
        + \mathbb{E}_{\mathbf{x}_1 \sim p_{1 \mid t}(\mathbf{x}_1 \mid \mathbf{x}_t, y)}\left[
            \left(
                \frac{\exp \left(\frac{1}{\beta} r(\mathbf{x}_1, y)\right)}
                {\mathbb{E}_{\mathbf{x}_1' \sim p_{1 \mid t}(\mathbf{x}_1 \mid \mathbf{x}_t, y)}
                    \left[\,\exp \left(\frac{1}{\beta} r(\mathbf{x}_1', y)\right)\,\right]} - 1
            \right)
      \boldsymbol{v}_t(\mathbf{x}_t \mid \mathbf{x}_1, y) \right]. \\[4pt] 
\end{aligned}
\]

The above derivation is the training-based guidance for flow matching, where we need to train the first guidance network $\boldsymbol{\psi}^*_1$ satisfies:
\[
h_{\boldsymbol{\psi}^*_1}\left(\mathbf{x}_t, y, t\right)=\mathbb{E}_{\mathbf{x}_1 \sim p_{1 \mid t}(\mathbf{x}_1 \mid \mathbf{x}_t, y)}
                    \left[\,\exp \left(\frac{1}{\beta} r(\mathbf{x}_1, y)\right)\,\right]
\]

by minimizing the objective  
{\small
\begin{equation*}
\mathcal{L}_{\text{guidance}}(\boldsymbol{\psi}_1) :=\mathbb{E}_{p(\mathbf{x}_1, \mathbf{x}_t, y)}\left[\left\|h_{\boldsymbol{\psi}_1}\left(\mathbf{x}_t, y, t\right)-\exp(\frac{1}{\beta}r\left(\mathbf{x}_1, y)\right)\right\|_2^2\right]. 
\end{equation*}}

And then we need the second guidance network $\boldsymbol{\psi}^*_2$ satisfies:
\[
h_{\boldsymbol{\psi}^*_2}\left(\mathbf{x}_t, y, t\right)=\mathbb{E}_{\mathbf{x}_1 \sim p_{1 \mid t}(\mathbf{x}_1 \mid \mathbf{x}_t, y)}\left[
            \left(
                \frac{\exp \left(\frac{1}{\beta} r(\mathbf{x}_1, y)\right)}
                {\mathbb{E}_{\mathbf{x}_1' \sim p_{1 \mid t}(\mathbf{x}_1 \mid \mathbf{x}_t, y)}
                    \left[\,\exp \left(\frac{1}{\beta} r(\mathbf{x}_1', y)\right)\,\right]} - 1
            \right)
      \boldsymbol{v}_t(\mathbf{x}_t \mid \mathbf{x}_1, y) \right]
\]
by minimizing the objective  
\begin{equation*} 
\mathcal{L}_{\text{guidance}}(\boldsymbol{\psi}_2) :=\mathbb{E}_{p(\mathbf{x}_1, \mathbf{x}_t, y)}\left[\left\|h_{\boldsymbol{\psi}_2}\left(\mathbf{x}_t, y, t\right)-\left(
                \frac{\exp \left(\frac{1}{\beta} r(\mathbf{x}_1, y)\right)}
                {h_{\boldsymbol{\psi}_1}\left(\mathbf{x}_t, y, t\right)} - 1
            \right)
      \boldsymbol{v}_t(\mathbf{x}_t \mid \mathbf{x}_1, y)\right\|_2^2\right]. 
\end{equation*}
The guidance network for flow matching is more complex than that used in diffusion models. The estimation errors from two guidance networks may accumulate and ultimately degrade generation performance. To address this limitation, we propose a training-free guidance method for flow matching that mitigates these issues.
\begin{align}
&\boldsymbol{v}_t^q(\mathbf{x}_t, y) \nonumber \\
&= \boldsymbol{v}_t^p(\mathbf{x}_t, y)
+ \E_{\mathbf{x}_1 \sim p_{1|t}(\mathbf{x}_1 \mid \mathbf{x}_t, y)}
\Bigg[
\left(
\frac{
\exp\!\left(\frac{1}{\beta} r(\mathbf{x}_1, y)\right)
}{
\E_{\mathbf{x}_1' \sim p_{1|t}(\mathbf{x}_1' \mid \mathbf{x}_t, y)}
\!\left[\exp\!\left(\frac{1}{\beta} r(\mathbf{x}_1', y)\right)\right]
}
-1
\right)
\boldsymbol{v}_t(\mathbf{x}_t \mid \mathbf{x}_1, y)
\Bigg] \nonumber \\[4pt]
&= \boldsymbol{v}_t^p(\mathbf{x}_t, y)
+ \int_{\mathbf{x}_1}
\left(
\frac{
\exp\!\left(\frac{1}{\beta} r(\mathbf{x}_1, y)\right)
}{
\E_{\mathbf{x}'_1 \sim p_{1|t}}
\!\left[\exp\!\left(\frac{1}{\beta} r(\mathbf{x}'_1, y)\right)\right]
}
- 1
\right)
\boldsymbol{v}_t(\mathbf{x}_t \mid \mathbf{x}_1, y)\,
p_{1|t}(\mathbf{x}_1 \mid \mathbf{x}_t, y)\,
d\mathbf{x}_1 \nonumber \\[4pt]
&= \boldsymbol{v}_t^p(\mathbf{x}_t, y)
+ \int_{\mathbf{x}_1}
\left(
\frac{
\exp\!\left(\frac{1}{\beta} r(\mathbf{x}_1, y)\right)
}{
\E_{\mathbf{x}_1 \sim p_{1|t}}
\!\left[\exp\!\left(\frac{1}{\beta} r(\mathbf{x}_1, y)\right)\right]
}
- 1
\right)
\boldsymbol{v}_t(\mathbf{x}_t \mid \mathbf{x}_1, y)\,
\frac{
p_{t|1}(\mathbf{x}_t \mid \mathbf{x}_1, y)\,
p(\mathbf{x}_1 \mid y)
}{
p_t(\mathbf{x}_t \mid y)
}
d\mathbf{x}_1 \nonumber \\[4pt]
&= \boldsymbol{v}_t^p(\mathbf{x}_t, y)
+ \E_{\mathbf{x}_1 \sim p(\mathbf{x}_1 \mid y)}
\Bigg[
\left(
\frac{
\exp\!\left(\frac{1}{\beta} r(\mathbf{x}_1, y)\right)
}{
\E_{\mathbf{x}_1 \sim p_{1|t}}
\!\left[\exp\!\left(\frac{1}{\beta} r(\mathbf{x}_1, y)\right)\right]
}
- 1
\right)
\boldsymbol{v}_t(\mathbf{x}_t \mid \mathbf{x}_1, y)\,
\frac{
p_{t|1}(\mathbf{x}_t \mid \mathbf{x}_1, y)
}{
p_t(\mathbf{x}_t \mid y)
}
\Bigg] \nonumber \\[4pt]
&= \boldsymbol{v}_t^p(\mathbf{x}_t, y)
+ \E_{\mathbf{x}_1 \sim p(\mathbf{x}_1 \mid y)}
\Bigg[
\left(
\frac{
\exp\!\left(\frac{1}{\beta} r(\mathbf{x}_1, y)\right)
}{
\E_{\mathbf{x}_1 \sim p_{1|t}}
\!\left[\exp\!\left(\frac{1}{\beta} r(\mathbf{x}_0, y)\right)\right]
}
- 1
\right)
\boldsymbol{v}_t(\mathbf{x}_t \mid \mathbf{x}_1, y)\,
\frac{
p_{t|1}(\mathbf{x}_t \mid \mathbf{x}_1, y)
}{
\E_{\mathbf{x}_1 \sim p(\mathbf{x}_1\mid y)}
\!\left[p_{t|1}(\mathbf{x}_t \mid \mathbf{x}_1, y)\right]
}
\Bigg] \nonumber \\[4pt]
&= \boldsymbol{v}_t^p(\mathbf{x}_t , y)
+ \E_{\mathbf{x}_1 \sim p(\mathbf{x}_1 \mid y)}
\Bigg[
\left(
\frac{
\exp\!\left(\frac{1}{\beta} r(\mathbf{x}_1, y)\right)
}{
\E_{\mathbf{x}_1 \sim p(\mathbf{x}_1 \mid y)}
\!\left[
\exp\!\left(\frac{1}{\beta} r(\mathbf{x}_1, y)\right)
\dfrac{
p_{t|1}(\mathbf{x}_t \mid \mathbf{x}_1, y)
}{
\E_{\mathbf{x}_1 \sim p(\mathbf{x}_1\mid y)}
\!\left[p_{t|1}(\mathbf{x}_t \mid \mathbf{x}_1, y)\right]
}
\right]
}
- 1
\right) \notag\\
& \hspace{250pt}\boldsymbol{v}_t(\mathbf{x}_t \mid \mathbf{x}_1, y)\,
\frac{
p_{t|1}(\mathbf{x}_t \mid \mathbf{x}_1, y)
}{
\E_{\mathbf{x}_1 \sim p(\mathbf{x}_1 \mid y)}
\!\left[p_{t|1}(\mathbf{x}_t \mid \mathbf{x}_1, y)\right]
}
\Bigg]. \notag
\end{align}
\end{proof}

\subsection{Proof of Lemma \ref{thm:exact_guidance}}
\begin{proof}

The proof is straightforward and we include it below for completeness. Note that the objective function can be rewritten as 
\[
\begin{aligned}
    &\mathcal{L}_{\text{guidance}}(\boldsymbol{\psi})    \\
    := & \mathbb{E}_{p(\mathbf{x}_0, \mathbf{x}_t, y)}\left[\left\|h_{\boldsymbol{\psi}}\left(\mathbf{x}_t, y, t\right)-\exp \left(\frac{1}{\beta} r(\mathbf{x}_0, y)\right)\right\|_2^2\right]\\
    = &   \int_{\mathbf{x}_t} \int_{y} \left\{\int_{\mathbf{x}_0}p(\mathbf{x}_0|\mathbf{x}_t,y) \left\|h_{\boldsymbol{\psi}}\left(\mathbf{x}_t, y, t\right)-\exp \left(\frac{1}{\beta} r(\mathbf{x}_0, y)\right)\right\|_2^2 d\mathbf{x}_0 \right\} p(\mathbf{x}_t|y) p(y) dy d\mathbf{x}_t \\
    = & \int_{\mathbf{x}_t} \int_{y} \left\{ \left\|h_{\boldsymbol{\psi}}(\mathbf{x}_t, y, t)\right\|_2^2  -  2 \langle h_{\boldsymbol{\psi}}(\mathbf{x}_t, y, t),  \int_{\mathbf{x}_0}p(\mathbf{x}_0|\mathbf{x}_t, y)  \exp \left(\frac{1}{\beta} r(\mathbf{x}_0, y)\right)    d\mathbf{x}_0 \rangle \right\} p(\mathbf{x}_t|y) p(y) dyd\mathbf{x}_t + C \\
    = &  \int_{\mathbf{x}_t} \int_{y} \left\|h_{\boldsymbol{\psi}}(\mathbf{x}_t, y, t) - \mathbb{E}_{p(\mathbf{x}_0 |\mathbf{x}_t, y)}\left[\exp \left(\frac{1}{\beta} r(\mathbf{x}_0, y)\right)\right] 
    \right\|_2^2 p(\mathbf{x}_t|y) p(y) dyd\mathbf{x}_t,
\end{aligned}
\]
where $C$ is a constant independent of $\boldsymbol{\psi}$. Thus we have the minimizer $\boldsymbol{\psi}^* = \underset{\boldsymbol{\psi}}{\arg \min } \ \mathcal{L}_{\text{guidance}}(\boldsymbol{\psi})$ satisfies $h_{\boldsymbol{\psi}^*}\left(\mathbf{x}_t, y, t\right)=\mathbb{E}_{p(\mathbf{x}_0|\mathbf{x}_t, y)}\left[{\exp \left(\frac{1}{\beta} r(\mathbf{x}_0, y)\right)}\right]$.
\end{proof}

\section{Gradient-free Designs for Diffusion Models}\label{apen:gradient-free}
After we learn the guidance network by Algorithm \ref{alg:guidance_training_regularized}, we can adopt  Eq \eqref{eq:dsm_IS_sampling} for inference. Although we can easily calculate the gradient of the guidance network with respect to $\mathbf{x}_t$ by autograd, a question is whether we can avoid the gradient calculation. In this section, we propose two additional designs for gradient-free guidance of diffusion models.

\subsection{Training-free guidance for Diffusion Models}
The first design is converting the gradient of the log expectation to the expectation under reward weighted distribution, and then we can apply a similar trick as training-free guidance that uses importance sampling to approximate the conditional expectation through Monte Carlo sampling under the marginal data distribution.

\begin{theorem}[Reward-Weighted Score Gradient]
\label{thm:reward_weighted_score}
Let $p(\mathbf{x}_0|\mathbf{x}_t)$ be the reverse diffusion posterior, $r(\mathbf{x}_0, y)$ be a reward function, and $\beta > 0$ be a temperature parameter. Define the reward-weighted distribution
\begin{equation}
\tilde{p}(\mathbf{x}_0|\mathbf{x}_t, y)=\frac{p(\mathbf{x}_0|\mathbf{x}_t, y) \exp\!\left(\tfrac{1}{\beta} r(\mathbf{x}_0, y)\right)}{\mathbb{E}_{p(\mathbf{x}_0|\mathbf{x}_t, y)}
\left[
\exp\!\left(\tfrac{1}{\beta} r(\mathbf{x}_0, y)\right)
\right]} . \notag
\end{equation}
Then the gradient of the log-partition function satisfies
\begin{align}
\nabla_{\mathbf{x}_t} \log 
\mathbb{E}_{p(\mathbf{x}_0|\mathbf{x}_t, y)}
\left[
\exp\!\left(\tfrac{1}{\beta} r(\mathbf{x}_0, y)\right)
\right]
=
\mathbb{E}_{\tilde{p}(\mathbf{x}_0|\mathbf{x}_t, y)}
\left[
\nabla_{\mathbf{x}_t}\log p(\mathbf{x}_0|\mathbf{x}_t, y)
\right]. \notag
\end{align}
Furthermore, this gradient can be approximated via importance sampling: given samples $\{\mathbf{x}_0^{(k)}\}_{k=1}^K \sim p(\mathbf{x}_0|y)$ from a proposal distribution,
\begin{align}
\nabla_{\mathbf{x}_t} \log 
\mathbb{E}_{p(\mathbf{x}_0|\mathbf{x}_t, y)}
\left[
\exp\!\left(\tfrac{1}{\beta} r(\mathbf{x}_0, y)\right)
\right]
\approx
\sum_{k=1}^{K} \tilde{w}_k \,
\nabla_{\mathbf{x}_t}\log p(\mathbf{x}_0^{(k)}|\mathbf{x}_t, y), \notag
\end{align}
where the normalized importance weights are
\begin{equation}
\tilde{w}_k
=
\frac{u_k}{\sum_{j=1}^{K} u_j}, 
\quad
u_k=
\frac{
p(\mathbf{x}_t|\mathbf{x}_0^{(k)}, y)\,
\exp\!\left(\tfrac{1}{\beta} r(\mathbf{x}_0^{(k)}, y)\right)
}{
\mathbb{E}_{p(\mathbf{x}_0|y)}
\left[p(\mathbf{x}_t|\mathbf{x}_0, y)\right]
}.\notag
\end{equation}
\end{theorem}

\begin{proof}[Proof of Theorem \ref{thm:reward_weighted_score}]
We first convert the gradient of the log expectation to the expectation under reward weighted distribution:
\begin{align}
&\nabla_{\mathbf{x}_t} \log 
\mathbb{E}_{p(\mathbf{x}_0|\mathbf{x}_t, y)}
\left[
\exp\!\left(\tfrac{1}{\beta} r(\mathbf{x}_0, y)\right)
\right] \notag \\
&=
\frac{
\nabla_{\mathbf{x}_t}
\mathbb{E}_{p(\mathbf{x}_0|\mathbf{x}_t, y)}
\left[
\exp\!\left(\tfrac{1}{\beta} r(\mathbf{x}_0, y)\right)
\right]
}{
\mathbb{E}_{p(\mathbf{x}_0|\mathbf{x}_t, y)}
\left[
\exp\!\left(\tfrac{1}{\beta} r(\mathbf{x}_0, y)\right)
\right]
} \tag{chain rule} \\
&=
\frac{
\nabla_{\mathbf{x}_t} \int p(\mathbf{x}_0|\mathbf{x}_t, y) \exp\!\left(\tfrac{1}{\beta} r(\mathbf{x}_0, y)\right) d\mathbf{x}_0
}{
\mathbb{E}_{p(\mathbf{x}_0|\mathbf{x}_t, y)}
\left[
\exp\!\left(\tfrac{1}{\beta} r(\mathbf{x}_0, y)\right)
\right]
} \tag{definition of expectation} \\
&=
\frac{
\int \nabla_{\mathbf{x}_t} p(\mathbf{x}_0|\mathbf{x}_t, y) \exp\!\left(\tfrac{1}{\beta} r(\mathbf{x}_0, y)\right) d\mathbf{x}_0
}{
\mathbb{E}_{p(\mathbf{x}_0|\mathbf{x}_t, y)}
\left[
\exp\!\left(\tfrac{1}{\beta} r(\mathbf{x}_0, y)\right)
\right]
} \tag{interchange $\nabla$ and $\int$} \\
&=
\frac{
\int p(\mathbf{x}_0|\mathbf{x}_t, y) \nabla_{\mathbf{x}_t} \log p(\mathbf{x}_0|\mathbf{x}_t, y) \exp\!\left(\tfrac{1}{\beta} r(\mathbf{x}_0, y)\right) d\mathbf{x}_0
}{
\mathbb{E}_{p(\mathbf{x}_0|\mathbf{x}_t, y)}
\left[
\exp\!\left(\tfrac{1}{\beta} r(\mathbf{x}_0, y)\right)
\right]
} \tag{log-derivative trick} \\
&=
\frac{
\mathbb{E}_{p(\mathbf{x}_0|\mathbf{x}_t, y)}
\left[
\nabla_{\mathbf{x}_t}\log p(\mathbf{x}_0|\mathbf{x}_t, y)\,
\exp\!\left(\tfrac{1}{\beta} r(\mathbf{x}_0, y)\right)
\right]
}{
\mathbb{E}_{p(\mathbf{x}_0|\mathbf{x}_t, y)}
\left[
\exp\!\left(\tfrac{1}{\beta} r(\mathbf{x}_0, y)\right)
\right]
} \tag{definition of expectation} \\
&=
\mathbb{E}_{\tilde{p}(\mathbf{x}_0|\mathbf{x}_t, y)}
\left[
\nabla_{\mathbf{x}_t}\log p(\mathbf{x}_0|\mathbf{x}_t, y)
\right], \tag{definition of $\tilde{p}$}
\end{align}
where the last equality follows from the fact that the ratio of expectations defines precisely the expectation under the normalized distribution $\tilde{p}(\mathbf{x}_0|\mathbf{x}_t, y)$.

For the importance sampling approximation, we rewrite the expectation using proposal distribution $q(\mathbf{x}_0) = p(\mathbf{x}_0|y)$:
\begin{align}
&\mathbb{E}_{\tilde{p}(\mathbf{x}_0|\mathbf{x}_t, y)}
\left[
\nabla_{\mathbf{x}_t}\log p(\mathbf{x}_0|\mathbf{x}_t, y)
\right] \notag \notag\\
&= \frac{
\int \tilde{p}(\mathbf{x}_0|\mathbf{x}_t, y) \, \nabla_{\mathbf{x}_t}\log p(\mathbf{x}_0|\mathbf{x}_t, y) \, d\mathbf{x}_0
}{1} \notag\\
&= \frac{
\int \frac{\tilde{p}(\mathbf{x}_0|\mathbf{x}_t, y)}{p(\mathbf{x}_0|y)} \, p(\mathbf{x}_0|y) \, \nabla_{\mathbf{x}_t}\log p(\mathbf{x}_0|\mathbf{x}_t, y) \, d\mathbf{x}_0
}{
\int \frac{\tilde{p}(\mathbf{x}_0|\mathbf{x}_t, y)}{p(\mathbf{x}_0|y)} \, p(\mathbf{x}_0|y) \, d\mathbf{x}_0
} \notag\\
&= \frac{
\mathbb{E}_{p(\mathbf{x}_0|y)}
\left[
\frac{\tilde{p}(\mathbf{x}_0|\mathbf{x}_t, y)}{p(\mathbf{x}_0|y)} \, \nabla_{\mathbf{x}_t}\log p(\mathbf{x}_0|\mathbf{x}_t, y)
\right]
}{
\mathbb{E}_{p(\mathbf{x}_0|y)}
\left[
\frac{\tilde{p}(\mathbf{x}_0|\mathbf{x}_t, y)}{p(\mathbf{x}_0|y)}
\right]
} \notag\\
&\approx \frac{
\sum_{k=1}^K u_k \, \nabla_{\mathbf{x}_t}\log p(\mathbf{x}_0^{(k)}|\mathbf{x}_t, y)
}{
\sum_{k=1}^K u_k
}\notag\\
&= \sum_{k=1}^K \tilde{w}_k \, \nabla_{\mathbf{x}_t}\log p(\mathbf{x}_0^{(k)}|\mathbf{x}_t, y),
\end{align}
where the approximation uses Monte Carlo sampling with $\mathbf{x}_0^{(k)} \sim p(\mathbf{x}_0|y)$ and the unnormalized importance weights $u_k$ are
\begin{align}
u_k &=
\frac{
p(\mathbf{x}_0^{(k)}|\mathbf{x}_t, y)\,
\exp\!\left(\tfrac{1}{\beta} r(\mathbf{x}_0^{(k)}, y)\right)
}{
p(\mathbf{x}_0^{(k)}|y)
}\notag \\
&=\frac{
p(\mathbf{x}_0^{(k)}|\mathbf{x}_t, y) p(\mathbf{x}_t|y)\,
\exp\!\left(\tfrac{1}{\beta} r(\mathbf{x}_0^{(k)}, y)\right)
}{
p(\mathbf{x}_0^{(k)}|y)p(\mathbf{x}_t|y)
} \tag{multiply by $p(\mathbf{x}_t|y)$} \\ 
&=\frac{p(\mathbf{x}_t|\mathbf{x}_0^{(k)}, y)\,
\exp\!\left(\tfrac{1}{\beta} r(\mathbf{x}_0^{(k)}, y)\right)}{p(\mathbf{x}_t|y)} \tag{Bayes' rule} \\
&=\frac{p(\mathbf{x}_t|\mathbf{x}_0^{(k)}, y)\,
\exp\!\left(\tfrac{1}{\beta} r(\mathbf{x}_0^{(k)}, y)\right)}{\int p(\mathbf{x}_t, \mathbf{x}_0|y)d\mathbf{x}_0} \tag{marginalization} \\
&=\frac{p(\mathbf{x}_t|\mathbf{x}_0^{(k)}, y)\,
\exp\!\left(\tfrac{1}{\beta} r(\mathbf{x}_0^{(k)}, y)\right)}{\int p(\mathbf{x}_t|\mathbf{x}_0, y) p(\mathbf{x}_0|y) d\mathbf{x}_0} \tag{chain rule} \\
&=\frac{p(\mathbf{x}_t|\mathbf{x}_0^{(k)}, y)\,
\exp\!\left(\tfrac{1}{\beta} r(\mathbf{x}_0^{(k)}, y)\right)}{\mathbb{E}_{p(\mathbf{x}_0|y)}
\left[p(\mathbf{x}_t|\mathbf{x}_0, y)\right]}. \tag{definition of expectation}
\end{align}

And the $\nabla_{\mathbf{x}_t}\log p(\mathbf{x}_0^{(k)}|\mathbf{x}_t, y)$ can be easily computed through Bayesian rule. 
\begin{align}
    \nabla_{\mathbf{x}_t}\log p(\mathbf{x}_0^{(k)}|\mathbf{x}_t, y)&=\nabla_{\mathbf{x}_t}\log \frac{p(\mathbf{x}_t|\mathbf{x}_0^{(k)}, y)p(\mathbf{x}_0^{(k)}|y)}{p(\mathbf{x}_t|y)} \notag\\
    &=\nabla_{\mathbf{x}_t}\log p(\mathbf{x}_t|\mathbf{x}_0^{(k)}, y)-\nabla_{\mathbf{x}_t}\log p(\mathbf{x}_t|y) \notag
\end{align}
\end{proof}

\subsection{Gradient-free Finetuning-free Guidance for Diffusion Models}
Another method is directly fitting a neural network to estimate the guidance term.
\begin{theorem}
\label{thm:exact_guidance_no_grad}
For a neural network $h_{\boldsymbol{\psi}}\left(\mathbf{x}_t, y, t\right)$ parameterized by $\boldsymbol{\psi}$, define the objective  

\begin{equation} \label{eq:guidance_no_grad}
\small{\begin{aligned}
&\mathcal{L}_{\text{guidance}}^*(\boldsymbol{\psi}) :=\mathbb{E}_{p(\mathbf{x}_0, \mathbf{x}_T, y)}\Bigg[ 
    \frac{1}{\mathbb{E}_{p(\mathbf{x}_0 |\mathbf{x}_T, y)}\left[\exp \left(\frac{1}{\beta} r(\mathbf{x}_0, y)\right)\right]} \notag\\
    &\frac{1}{p(\mathbf{x}_0|\mathbf{x}_T, y)}
    \left\| h_{\boldsymbol{\psi}}\left(\mathbf{x}_T, y\right) 
    - \nabla_{\mathbf{x}_T} p(\mathbf{x}_0|\mathbf{x}_T, y) 
    \exp\left(\frac{1}{\beta} r(\mathbf{x}_0, y)\right) \right\|_2^2 
\Bigg],
\end{aligned}}
\end{equation}

then its minimizer $\boldsymbol{\psi}^* = \underset{\boldsymbol{\psi}}{\arg \min } \ \mathcal{L}_{\text{guidance}}(\boldsymbol{\psi})$ satisfies:
\[
h_{\boldsymbol{\psi}^*}\left(\mathbf{x}_T, y\right)=\nabla_{\mathbf{x}_T} \log\mathbb{E}_{p(\mathbf{x}_0 |\mathbf{x}_T, y)}\left[{\exp(\frac{1}{\beta}r(\mathbf{x}_0, y))}\right].
\]
\end{theorem}

\begin{proof}[Proof of Theorem \ref{thm:exact_guidance_no_grad}]
    \[
\begin{aligned}
& \int_{\mathbf{x}_T}  \left\|h_{\boldsymbol{\psi}}(\mathbf{x}_T, y) - \nabla_{\mathbf{x}_T} \log \mathbb{E}_{p(\mathbf{x}_0 |\mathbf{x}_T)}\left[\exp \left(\frac{1}{\beta} r(\mathbf{x}_0, y)\right)\right] 
    \right\|_2^2 p(\mathbf{x}_T)d\mathbf{x}_T  \\
=&\int_{\mathbf{x}_T} \left\{ \left\|h_{\boldsymbol{\psi}}(\mathbf{x}_T, y)\right\|_2^2  -  2 \langle h_{\boldsymbol{\psi}}(\mathbf{x}_T, y),  \nabla_{\mathbf{x}_T} \log \int_{\mathbf{x}_0}p(\mathbf{x}_0|\mathbf{x}_T)  \exp \left(\frac{1}{\beta} r(\mathbf{x}_0, y)\right)    d\mathbf{x}_0 \rangle \right\} p(\mathbf{x}_T)d\mathbf{x}_T + C \\
=&\int_{\mathbf{x}_T} \left\{ \left\|h_{\boldsymbol{\psi}}(\mathbf{x}_T, y)\right\|_2^2  -  2 \langle h_{\boldsymbol{\psi}}(\mathbf{x}_T, y),  \frac{\nabla_{\mathbf{x}_T} \int_{\mathbf{x}_0}p(\mathbf{x}_0|\mathbf{x}_T)  \exp \left(\frac{1}{\beta} r(\mathbf{x}_0, y)\right)    d\mathbf{x}_0}{\int_{\mathbf{x}_0}p(\mathbf{x}_0|\mathbf{x}_T)  \exp \left(\frac{1}{\beta} r(\mathbf{x}_0, y)\right)    d\mathbf{x}_0} \rangle \right\} p(\mathbf{x}_T)d\mathbf{x}_T + C \\
=& \int_{\mathbf{x}_T} 
\frac{1}{
\int_{\mathbf{x}_0} p(\mathbf{x}_0|\mathbf{x}_T) 
\exp \left(\frac{1}{\beta} r(\mathbf{x}_0, y)\right) d\mathbf{x}_0
}
\left\{
\left\|h_{\boldsymbol{\psi}}(\mathbf{x}_T, y)\right\|_2^2 
- 2 \, 
\Big\langle 
h_{\boldsymbol{\psi}}(\mathbf{x}_T, y),  
\int_{\mathbf{x}_0} \nabla_{\mathbf{x}_T} p(\mathbf{x}_0|\mathbf{x}_T) 
\exp \left(\frac{1}{\beta} r(\mathbf{x}_0, y)\right) d\mathbf{x}_0 
\Big\rangle
\right\} \\
& \hspace{390pt}  p(\mathbf{x}_T)d\mathbf{x}_T + C\\
= &\int_{\mathbf{x}_T}  \left\{
\int_{\mathbf{x}_0} p(\mathbf{x}_0|\mathbf{x}_T)
\frac{1}{
\int_{\mathbf{x}_0} p(\mathbf{x}_0|\mathbf{x}_T)
\exp \left(\frac{1}{\beta} r(\mathbf{x}_0, y)\right)
d\mathbf{x}_0
} \right. \\
&\hspace{100pt}\left.
\frac{1}{p(\mathbf{x}_0|\mathbf{x}_T)}
\left\|
h_{\boldsymbol{\psi}}\left(\mathbf{x}_T, y\right)
-
\nabla_{\mathbf{x}_T} p(\mathbf{x}_0|\mathbf{x}_T)
\exp \left(\frac{1}{\beta} r(\mathbf{x}_0, y)\right)
\right\|_2^2
\, d\mathbf{x}_0
\right\}
p(\mathbf{x}_T)d\mathbf{x}_T + C \\
=& \mathbb{E}_{p(\mathbf{x}_0, \mathbf{x}_T)}\left[ \frac{1}{\int_{\mathbf{x}_0}p(\mathbf{x}_0|\mathbf{x}_T)  \exp \left(\frac{1}{\beta} r(\mathbf{x}_0, y)\right)    d\mathbf{x}_0} \frac{1}{p(\mathbf{x}_0|\mathbf{x}_T)}\left\|h_{\boldsymbol{\psi}}\left(\mathbf{x}_T, y\right)-\nabla_{\mathbf{x}_T} p(\mathbf{x}_0|\mathbf{x}_T)\exp \left(\frac{1}{\beta} r(\mathbf{x}_0, y)\right)\right\|_2^2\right] +C \\
:=& \mathcal{L}_{\text{guidance}}(\boldsymbol{\psi}) + C,
\end{aligned}
\]
where $\nabla_{\mathbf{x}_T} p(\mathbf{x}_0|\mathbf{x}_T)$ can be computed by $\nabla_{\mathbf{x}_T} \log p(\mathbf{x}_0|\mathbf{x}_T)\frac{1}{p(\mathbf{x}_0|\mathbf{x}_T)}$.
\end{proof}

\section{More Details on Experiments} \label{ap:ablation}

\subsection{Algorithms for Training the Guidance Network}

Algorithm \ref{alg:guidance_training_regularized} is the algorithm for training the guidance network.
\begin{algorithm}[htb]
    \centering
    \caption{Algorithm for Training a Guidance Network}
    \label{alg:guidance_training_regularized}
    \begin{algorithmic}[1]
        \REQUIRE Samples from alignment dataset, pre-trained one-step diffusion model $s(\mathbf{x}_T, y, T)$, pre-determined reward function $r(\mathbf{x}_0, y)$, hyperparameters $\eta, \beta$, and initial weights of guidance network $\boldsymbol{\psi}$.
        \REPEAT
        \STATE Sample mini-batch data from alignment dataset with batch size $b$.
        \STATE Perturb $\mathbf{x}_0$ using forward transition $p(\mathbf{x}_T |\mathbf{x}_0)$.
        \STATE Compute guidance loss:
        \STATE {\small
        \begin{align*}
        \mathcal{L}_{\text{guidance}}(\boldsymbol{\psi}) =\frac{1}{b} \sum_{\mathbf{x}_0, \mathbf{x}_T, y} 
        \left\| h_{\boldsymbol{\psi}}\left(\mathbf{x}_T, y\right) - 
        \exp\left(\frac{1}{\beta}r(\mathbf{x}_0, y)\right) \right\|_2^2.
        \end{align*}}
        \STATE Sample mini-batch from winning responses $(\mathbf{x}', y)$ with batch size $b$.
        \STATE Perturb $\mathbf{x}'_0$ using forward transition $q(\mathbf{x}'_T |\mathbf{x}'_0)$.
        \STATE Compute consistency loss:
        \STATE {\small
        \begin{align*}
        \mathcal{L}_{\text{consistence}} = \frac{1}{b} \sum_{\mathbf{x}'_0, \mathbf{x}'_T, y} 
        \Big\| s(\mathbf{x}'_T, y, T) + \nabla_{\mathbf{x}'_T} \log h_{\boldsymbol{\psi}}(\mathbf{x}'_T, y)        - \nabla_{\mathbf{x}'_T} \log q(\mathbf{x}_T | \mathbf{x}'_0, y) \Big\|_2^2.
        \end{align*}}
        \STATE Update $\boldsymbol{\psi}$ via gradient descent:
        \begin{align*}
        \nabla_{\boldsymbol{\psi}} \left( \mathcal{L}_{\text{guidance}} + \eta \ \mathcal{L}_{\text{consistence}} \right).
        \end{align*}
        \UNTIL{convergence}
        \STATE \textbf{return} weights of guidance network $\boldsymbol{\psi}$.
    \end{algorithmic}
\end{algorithm}

\subsection{Ablation Study on Hyperparameter}
In this subsection, we provide the ablation study of the strength of the regularization $\eta$ and the strength of the reward function $\beta$ in the following table.

\begin{table}[h]
\centering
\caption{Ablation study of hyperparameter on PickScore.}
\vspace{5pt}
\begin{tabular}{c|c|c|c}
\toprule 
$\eta$ & $\beta=10$ & $\beta=15$ & $\beta=20$ \\
\midrule 
0.1 & 22.82 & 22.79 & 22.72 \\
\midrule 
0.5 & 22.78 & 23.01 & 22.79 \\
\midrule 
1   & 22.76 & \textbf{23.08} & 22.84 \\
\bottomrule 
\end{tabular}
\end{table}

\subsection{Prompts for Figure in Main Paper}\label{ap:prompt}

\begin{table}[h]
\caption{Prompts used to generate Figure \ref{fig:f1}.} 
\centering
\resizebox{\textwidth}{!}{ 
\begin{tabular}{|l|c|}
\hline
Image & Prompt \\
\hline
Col1 & Saturn rises on the horizon. \\
\hline
Col2 & a watercolor painting of a super cute kitten wearing a hat of flowers \\
\hline
Col3 & \begin{tabular}[c]{@{}l@{}}A galaxy-colored figurine floating over the sea at sunset, photorealistic.\end{tabular} \\
\hline
Col4 & \begin{tabular}[c]{@{}l@{}}fireclaw machine mecha animal beast robot of horizon forbidden west horizon zero dawn bioluminiscence, behance hd by jesper ejsing, by\\
rhads, makoto shinkai and lois van baarle, ilya kuvshinov, rossdraws global illumination\end{tabular} \\
\hline
Col5 & \begin{tabular}[c]{@{}l@{}}A swirling, multicolored portal emerges from the depths of an ocean of coffee, with waves of the rich liquid gently rippling outward. The \\
portal engulfs a coffee cup, which serves as a gateway to a fantastical dimension. The surrounding digital art landscape reflects the colors of \\
the portal, creating an alluring scene of endless possibilities.\end{tabular} \\
\hline
Col6 & A profile picture of an anime boy, half robot, brown hair \\
\hline
Col7 & \begin{tabular}[c]{@{}l@{}}Detailed Portrait of a cute woman vibrant pixie hair by Yanjun Cheng and Hsiao-Ron Cheng and Ilya Kuvshinov, medium close up, portrait \\
photography, rim lighting, realistic eyes, photorealism pastel, illustration\end{tabular} \\
\hline
Co18 & \begin{tabular}[c]{@{}l@{}}On the Mid-Autumn Festival, the bright full moon hangs in the night sky. A quaint pavilion is illuminated by dim lights, resembling a \\
beautiful scenery in a painting. Camera type: close-up. Camera lens type: telephoto. Time of day: night. Style of lighting: bright. Film type: \\
ancient style. HD.\end{tabular} \\
\hline
\end{tabular}
}
\end{table}

\end{document}